\documentclass[twoside,11pt]{article}

\usepackage{blindtext}

\usepackage{jmlr2e}

\makeatletter
\def\@starteditor{}\def\@endeditor{}\def\@editor{}
\makeatother

\usepackage{amsmath}
\usepackage{amssymb}

\let\jmlrproof\proof \let\endjmlrproof\endproof
\let\proof\relax \let\endproof\relax
   
\usepackage{amsthm}

\let\proof\jmlrproof \let\endproof\endjmlrproof
   
\usepackage{algorithmic}
\usepackage{algorithm}
\usepackage{comment}
\usepackage{array}
\usepackage{enumitem}
\usepackage{xcolor}
\usepackage{graphicx}
\usepackage{lipsum}
\usepackage[colorinlistoftodos]{todonotes}
\usepackage{float}
\usepackage{lastpage}
\usepackage{graphicx}
\usepackage{subcaption}
\usepackage{caption}
\usepackage{graphicx}
\usepackage{booktabs}
\usepackage{placeins}

\DeclareMathOperator{\Tr}{Tr}

\makeatletter
\floatstyle{ruled}
\newfloat{algorithm}{tbp}{loa}
\theoremstyle{plain}
\newtheorem{thm}{\protect\theoremname}
\theoremstyle{definition}

\ifx\proof\undefined
\newenvironment{proof}[1][\protect\proofname]{\par
	\normalfont\topsep6\p@\@plus6\p@\relax
	\trivlist
	\itemindent\parindent
	\item[\hskip\labelsep\scshape #1]\ignorespaces
}{%
	\endtrivlist\@endpefalse
}

\fi
\theoremstyle{plain}
\newtheorem{lem}[thm]{\protect\lemmaname}
\theoremstyle{definition}
\newtheorem{assum}[thm]{Assumption}

\theoremstyle{definition}
\providecommand{\algorithmname}{Algorithm}
\providecommand{\proofname}{Proof}
\providecommand{\definitionname}{Definition}
\providecommand{\lemmaname}{Lemma}
\providecommand{\theoremname}{Theorem}
\floatname{algorithm}{\protect\algorithmname}
\makeatother

\DeclareMathOperator{\Cov}{Cov}

\ShortHeadings{Distributed Training with Correlated Data}{Shitrit and Bistritz}
\firstpageno{1}

\begin{document}

\title{Cordial Learning: Distributed Training with Correlated Data}

\author{\name Sarah Shitrit \email sarahshitrit@mail.tau.ac.il \\
       \addr School of Electrical and Computer Engineering\\
       Tel Aviv University\\
       Tel Aviv, Israel 
       \AND
       \name Ilai Bistritz \email ilaibistritz@tauex.tau.ac.il \\
       \addr School of Industrial and Intelligent Systems Engineering and School of Electrical and Computer Engineering\\
       Tel Aviv University\\
       Tel Aviv, Israel }


\maketitle

\begin{abstract}
We consider a distributed learning task with agents that have correlated data. Specifically, the label of an agent depends on the input of other agents for the same sample, and these inputs are also correlated. Correlated data is the reality when agents share the same environment. Existing decentralized methods, such as federated learning, ignore the structure of the problem and perform poorly on correlated data. On the other hand, centralized approaches are infeasible due to privacy and communication constraints. We introduce cordial (correlated and distributed) learning to address this gap by sharing only low-dimensional outputs between the agents while training local models to extract informative signals from peers. This distributed learning induces a game in which the loss function of each agent depends on the models of others. Assuming a linear model, we prove that cordial learning converges with probability one to a globally optimal solution, despite the nonconvex global objective. Experiments on structured multi-digit MNIST tasks demonstrate that cordial learning remains highly effective even in highly nonlinear settings.
\end{abstract}
\begin{keywords}
  Federated Learning, Distributed Training, Game Theory, Collaborative Learning, Stochastic Optimization
\end{keywords}
\section{Introduction}

When deployed, ML systems rarely operate in a vacuum. Their data comes from a shared environment with many ML systems. Thus, the data of many ML systems can be highly correlated. Current
ML architectures are trained in solitude, ignoring any interaction they may encounter after deployment.

In distributed learning, and federated learning (FL) in particular, all edge devices (agents) face the same learning task. The data collected by an edge device is valuable to all others since it extends their effective training dataset. FL then avoids sending explicit data to the server, maintaining user privacy, while still allowing agents to learn from each other's data indirectly.

When ML agents share an environment, one agent's data is helpful not only because more data is better, but mostly because it shows the agent the part of the bigger picture that it cannot see alone. Technically, it means that the label of one agent depends on the input of others. Even if the agents have different ML tasks, so they cannot learn directly from others' data, the data of others contains information that explains their label better. 

The current distributed learning paradigm cannot benefit from the correlated data of ML agents that share an environment after deployment. In this paper, we introduce a new paradigm we term cordial learning (CL), since it is  \textbf{cor}related and \textbf{di}stributed. Similar to FL, CL is concerned with edge devices that each have a private dataset. However, the $l$-th datapoint of all devices is correlated (input and label), for every index $l$. 

CL is the practical reality in the following example applications: 

\textbf{Multi-Digit MNIST:} Each classifier sees a handwritten digit as input. However, the digits of all agents are dependent (e.g., credit card numbers, error correction codes, or the same handwriting). For example, adjacent digits must add up to 10, or all digits tend to equal a common hidden digit.

\textbf{Remote Sensors:} The agents are sensors that collect local environmental features (e.g., temperature, vibration, gas). The task is to classify whether a global phenomenon, such as a fire or earthquake, is occurring. Since all sensors in the same region are exposed to the same event, their data (inputs and labels) are correlated. However, each sensor only observes part of the full picture. Sharing model outputs, rather than raw data, allows the agents to benefit from the collective information while meeting the communication constraints.

\textbf{Medical Wearables in Shared Environments:}  
Each agent monitors a different patient using wearable sensors that record local physiological signals (e.g., heart rate, skin temperature). The task is to detect the onset of a disease. When patients are nearby, such as on a hospital ward or during an outbreak, their health outcomes are correlated, as infection spreads. In this case, leveraging cross-patient correlations can improve prediction accuracy. However, direct data sharing is prohibited due to privacy constraints. A distributed learning approach that shares only model outputs can enable collaboration without compromising sensitive health data.

\textbf{Formation Control:}  
A group of moving agents often has to maintain a target formation. Examples are platoons of autonomous vehicles or drones. Each vehicle collects local sensor data (e.g., LiDAR, radar, velocity) and must classify whether to accelerate, maintain speed, or decelerate. Since vehicles influence each other's decisions, especially those of their neighbors, their data is correlated. However, each vehicle observes only partial information. Sharing embedded outputs rather than raw sensory data allows collaborative real-time decision-making while preserving autonomy.

Our goal is to design a distributed
training scheme that achieves close to optimal performance in this cordial learning scenario, while maintaining low communication overhead. We measure performance by the sum of loss functions of all agents. 

A centralized architecture would collect data from all agents and train one large model to fit their combined input to their combined labels. Then, the central server is needed in real-time, so all devices can send their input data to the server and receive the result back. Such a centralized architecture requires significant communication between a central server and the agents to collect all their data, both in training and in real-time. Then the central server must train a huge ML model with a large dataset to obtain all agent models. Perhaps even worse, collecting user data in the central server can violate user
privacy. It also creates a single point of failure. Communication, latency, complexity, and privacy are known issues with any centralized architecture. In the context of cordial learning, the centralized architecture is also data-inefficient. The classification
tasks are inherently distributed in the sense that each agent
observes a different part of the whole problem. A centralized agent
will waste a lot of data \textquotedblleft understanding\textquotedblright{} the distributed nature of the problem (e.g., that there are $N$ digits
in the multi-digit example). In other words, since the problem is distributed, a
data-efficient architecture will also have to be
distributed.

In CL, the learning tasks of the agents are intertwined. If two agents each observe a digit, knowing that both digits are even or odd together, then each agent only needs the other agent to detect the parity of its digit. However, being able to classify a digit is very much related to classifying its parity. On the other hand, only sending the decision is not enough information. If an agent believes that the digit is six, but digits nine and three together are more likely, then this information is lost. A similar tradeoff exists in remote sensors that detect fire - even if one sensor is sure that currently there is no fire in its region, some indicators can help sensors in more sensitive areas to lower their detection threshold. 

We propose using an embedded output as the basic communication unit between the agents. The embedding is in the dimension of the output, a short vector compared to the input vectors. This embedded output is similar to the softmax output of a neural network (NN), but not the same. The softmax output of our agents is computed after incorporating the information from other agents, while the embedded output is processed before.

In CL, agents maintain coordination layers dedicated to extracting useful information from the inputs they receive from other agents for the \textit{same} sample. Thus, each agent simultaneously trains its NN along with these additional layers. This introduces a unique coupling (or interaction) between the agents. The weights of the NN of agent $n$ affect not only the loss of agent $n$, but also the loss of all other agents, since it changes the information they receive from agent $n$. As our analysis reveals, this interaction requires careful design. The sum of loss functions is nonconvex even in the linear regression case. Thus, many suboptimal stationary points may exist. 

For the linear theoretical model, we show that if the learning of the ``coordination layers" is slower than the learning of the embedding layers, then CL is guaranteed to converge to the optimal agent models that minimize the nonconvex sum of loss functions. This happens even though each agent locally optimizes only its loss function, ignoring the effect of its model on the loss of others. This makes CL a scalable, distributed training algorithm that can incorporate many devices, with little communication overhead.

\section{Previous Work}
Our work belongs to the growing literature on on-device learning. When the datasets of different agents are independent, each agent has a disjoint dataset for the same classification task. This corresponds to the classic on-device learning setting addressed by federated learning \citep{Kairouz2021,mitra2021linear,koloskova2022sharper}, where no inter-agent correlation exists. Reducing the communication of federated learning and distributed SGD has been recognized as crucial \citep{wang2021cooperative}. Unlike CL, FL is not designed to deal with correlated data. Additionally, FL requires that all devices have the same architecture so that the weights can be averaged. In CL, different devices can have different architectures. 

In \cite{bistritz2020distillation}, we designed a distributed distillation
scheme in which agents share softmax outputs on common unlabeled data to learn from each other's private data without revealing it. Since softmax is a short vector, the resulting communication
overhead is dramatically lower than with FL, and devices with different ML architectures can participate in the training. Distributed distillation does not deal with correlated data. In CL, agents share embedded outputs to learn from other agents about the \textit{same} sample they are currently learning from. 

Game theory and learning have recently become intertwined. The literature
on learning in games studies the interaction between players who learn
to maximize their reward. Then, convergence to a Nash equilibrium (NE) is guaranteed for various distributed algorithms in many scenarios \citep{xu2025learning,ratliff2016characterization,bistritz2022no, chasnov2020convergence,mazumdar2020gradient}. We study the opposite - how games
can help in learning. Game theory has already been proven insightful for learning in generative adversarial networks \citep[GANs;][]{goodfellow2020}.
Unlike the zero-sum games studied in GANs, our players are cooperative and each is an individual ML system \citep{bistritz2020cooperative}. 

Distributed learning with linear models that observe local measurements
with correlated noise was studied in \cite{Hong2022}. The risk functions
include a term that penalizes for deviating from neighboring models.
In our scenario, players do not necessarily converge to the same model.
On the contrary, players may face non-identically distributed data
and need to take heterogeneous actions (e.g., the medical wearables example above). 

Performative prediction studies a game scenario where the actions
of one learner affect both the risk scores and the distributions of
the data of the others \citep{Narang2023,linplug,perdomo2020performative}. The goal is to compute
an equilibrium that predicts the behavior of selfish learners. In
contrast, our agents are cooperative. Furthermore, given the distribution,
the datapoints of different players in \cite{Narang2023} are independent, while ours are correlated. Finally, \cite{Narang2023} assumes that
players can observe each other's actions, which is
unacceptable in our case since it would mean that players can observe
the model parameters of everyone.

CL resembles vertical federated learning (VFL) \citep{liu2024vertical}. In VFL, the different agents have different nonoverlapping features of the \textit{same} sample. This means that the inputs of the agents are correlated. However, in VFL, there is one learning task and one label, which only the ``active agent" possesses. In contrast, in CL, each agent has its own local learning task and its own label, and the labels are correlated between agents. In this sense, VFL is a special case of CL, where all labels are identical and only one agent orchestrates the learning. 

Multitask learning (MTL) \citep{maurer2006bounds} is also concerned with intertwined learning tasks. However, MTL trains a centralized architecture and not distributed agents. Even more important is that MTL is trained for tasks that occur asynchronously in real-time. In contrast, CL trains distributed agents that learn from synchronous correlated data points in real-time.  

While successive classification shares the idea of interdependent models, it differs from our setting in several fundamental ways. In pipelines such as Faster R-CNN \citep{ren2015faster}, the region proposal network and the final classifier share a common feature extraction backbone, introducing architectural coupling. However, the overall computation proceeds sequentially: region proposals are generated and then classified. In contrast, CL involves a network of distributed agents, each owning a private task and dataset. These agents do not share a backbone or data. Instead, they coordinate through shared latent representations generated during training. Information flows bidirectionally, and learning is formulated as a nonconvex game where each agent’s model parameters affect others.

\begin{figure*}[t]
    \centering
    \subfloat[Centralized agent architecture.]
    {
        \includegraphics[width=0.9\linewidth]{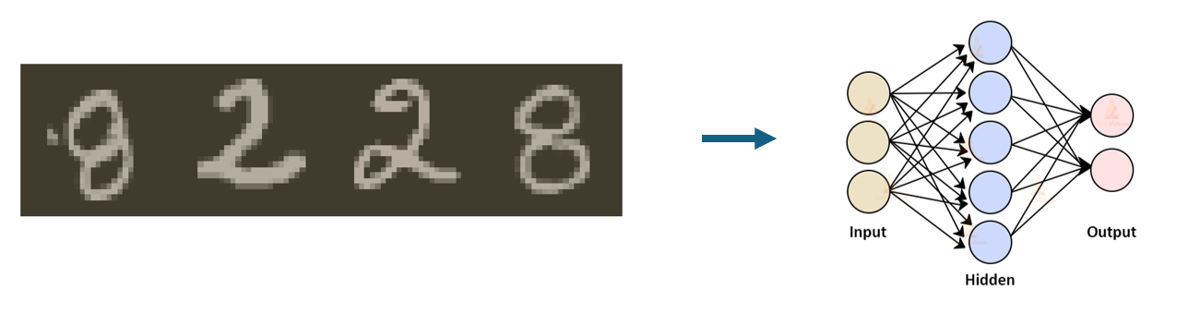}
        \label{subfig:centralized}
    }

    \subfloat[Isolated agents (no communication).]{
        \includegraphics[width=0.9\linewidth]{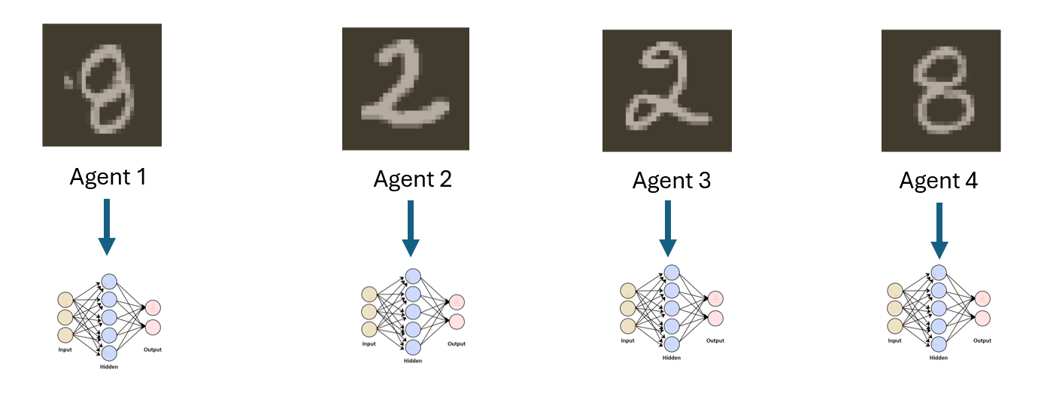}
        \label{subfig:isolated}
    }

    \subfloat[Cordial agents with shared information.]
    {
        \includegraphics[width=0.9\linewidth]{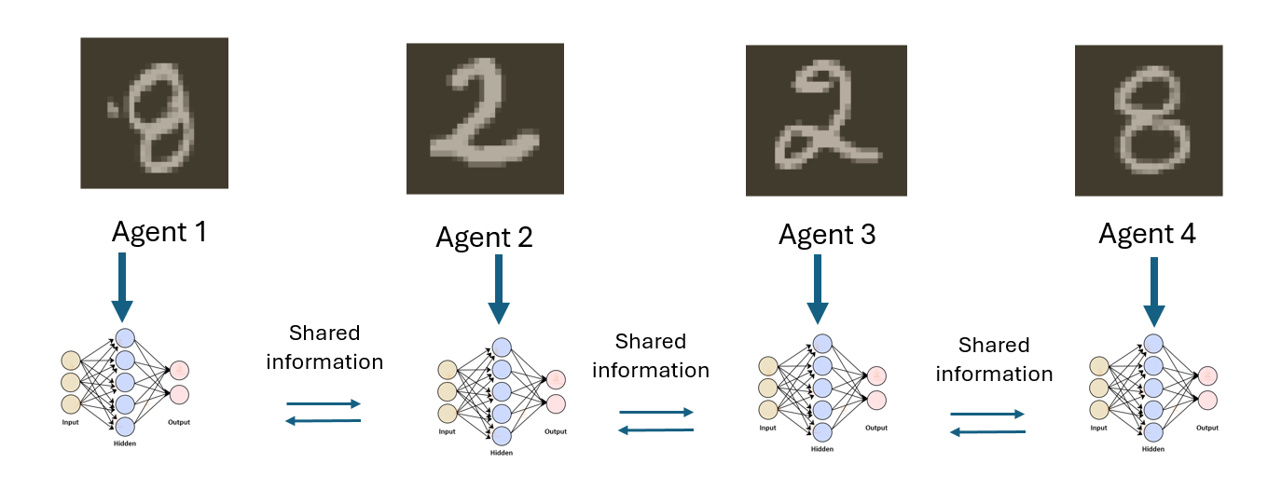}
        \label{subfig:cordial}
    }
    \caption{Different agent configurations for a digit classification task.}
    \label{fig:agent-configs}
\end{figure*}

\section{Problem Formulation}

We consider $N$ distributed agents, each on a different device.
Agent $n$ has a labeled training dataset $\mathcal{D}_{n}=\left\{ \boldsymbol{x}_{l}^{\left(n\right)},y_{l}^{\left(n\right)}\right\} _{l=1}^{D}$,
where $\boldsymbol{x}_{l}^{\left(n\right)}\in\mathbb{R}^{d}$
is an input and $y_{l}^{\left(n\right)}\in\mathbb{R}^{m}$ is a label (either for classification or regression). 
We assume that $\left[\boldsymbol{x}_{l}^{\left(n\right)},y_{l}^{\left(n\right)}\right]_{n=1}^{N}$
is the $l$-th sample of a random vector $\left[\boldsymbol{X}_{n},Y_{n}\right]_{n=1}^{N}$.
The crux of our model is that $\left(\boldsymbol{X}_{1},Y_{1}\right),\ldots,\left(\boldsymbol{X}_{N},Y_{N}\right)$
are statistically dependent. The datasets of all agents are
then $D$ independent samples of the vector $\left[\boldsymbol{X}_{n},Y_{n}\right]_{n=1}^{N}$.
This can be thought of as one large feature vector $\boldsymbol{x}_{l}=\left(\boldsymbol{x}_{l}^{\left(1\right)},\ldots,\boldsymbol{x}_{l}^{\left(N\right)}\right)$
that is distributed across agents, which is reminiscent of the idea of vertical federated learning \citep{liu2024vertical}. However, in cordial learning, each agent has their own learning task with their own label vector.

Ideally, to train the $N$ agents, we would wish to minimize the following empirical risk 
\begin{equation}
L\left(\boldsymbol{\omega}_{1},\ldots,\boldsymbol{\omega}_{N}\right)=\sum_{n=1}^{N}\sum_{l=1}^{D}\ell\left(\boldsymbol{\omega}_{n},\left[\boldsymbol{x}_{l}^{\left(m\right)},y_{l}^{\left(m\right)}\right]_{m=1}^{N}\right)\label{eq:1}
\end{equation}
for a loss function $\ell$ (e.g., MSE), where $\boldsymbol{\omega}_{n}$
is the model parameters of agent $n$. 

A centralized architecture to solve (\ref{eq:1}) treats the $N$
agents as one big model $\boldsymbol{\omega}=\left(\boldsymbol{\omega}_{1},\ldots,\boldsymbol{\omega}_{N}\right)$
where the input is $\boldsymbol{x}_{l}=\left(\boldsymbol{x}_{l}^{\left(1\right)},\ldots,\boldsymbol{x}_{l}^{\left(N\right)}\right)$
and the output is ${y}_{l}=\left({y}_{l}^{\left(1\right)},\ldots,{y}_{l}^{\left(N\right)}\right)$.
However, each agent $n$ can only observe its local data $\left(\boldsymbol{x}_{l}^{\left(n\right)},y_{l}^{\left(n\right)}\right)$
and not $\left(\boldsymbol{x}_{l}^{\left(m\right)},y_{l}^{\left(m\right)}\right)$
for $m\neq n$. For this reason, the traditional approach is just
to train each model separately, such that the $n$-th agent minimizes its loss function:
\begin{equation}
L^{\left(n\right)}\left(\boldsymbol{\omega}_{n}\right)=\sum_{l=1}^{D}\ell\left(\boldsymbol{\omega}_{n},\boldsymbol{x}_{l}^{\left(n\right)},y_{l}^{\left(n\right)}\right).\label{eq:2}
\end{equation}

\subsection{Linear Mean-Squared Error Model}
In general ML scenarios, the loss functions can be highly nonconvex and intractable. To capture the essence of the correlated data scenario and analyze the theoretical performance guarantees of cordial learning, we adopt a linear model. In this model, the random label vector of agent $i$ is a linear mixture of the input of all agents: 
\begin{equation}
    Y_i = A_iX_i + \sum_{j=1,j\neq i}^{N}B^{i}_{j}X_j + \nu_i\label{eq:3}
\end{equation}
where $A_{i}\in\mathbb{R}^{m\times d}$ and $B^{i}_{j} \in\mathbb{R}^{m\times d}$ are the unknown distribution parameters, and $\nu_i \in \mathbb{R}^{m}$ is a zero-mean noise with variance $\mathbb{E}\left\{ \|\nu_i\|^2\right\}=\sigma^{2}$, that is independent of all $\{X_i\}$.

We assume without loss of generality that $\mathbb{E}\left\{ X_{i}\right\} =0$ and $\mathbb{E}\left\{ X_{i}X_{i}^{T}\right\} = I$  for all $i$. As usual, each agent can preprocess the inputs to satisfy these assumptions. We also assume that $\mathbb{E}\left\{\|X_{i}\|^{4}\right\} < \infty$ for all $i$.

The naive approach to this problem is to ignore the correlation between the agents and train each of them separately, using the following loss function for agent $i$:

\begin{equation}
L_{i}(M_i)=\mathbb{E}\left\{ \left\Vert Y_{i}-M_{i}X_{i}\right\Vert ^{2}\right\} \label{eq:4}
\end{equation}
where the norm is the Frobenius norm $\left\Vert A\right\Vert =\sqrt{\Tr(AA^{T})}$ (throughout the paper). 

Such a training scheme is suboptimal. On the other extreme, we could train all agents in a centralized manner, which is encoded in the following loss function for agent $i$:

\begin{equation}
    L_{i}(M_{1}^{i},\ldots,M_{N}^{i})=\mathbb{E}\left\{ \left\Vert Y_{i}-\sum_{j=1}^{N}M_{j}^{i}X_{j}\right\Vert ^{2}\right\} \label{eq:5}
\end{equation}
so the model parameters are given by the matrices $M_j^{i}$ for all $j$. This is equivalent to a large learning model that faces a prediction task with an $mN$-dimensional output. By construction, the centralized approach can achieve the minimal training loss. However, centralized training will require each agent to share all its data with all others. Then the central server will still be needed in real-time to collect the inputs from all devices, process them, and send back the results to the agents. The communication overhead of such sharing is infeasible in large networks, especially in real-time. Moreover, sharing explicit data violates user privacy. Last but not least, when the task structure is distributed (e.g., remote fire sensors or medical sensors), a distributed architecture is more data-efficient than a centralized one that wastes data to understand the distributed structure of the problem. 

Our goal is to bypass the issues of a centralized training scheme while improving performance over the naive isolated scheme. To that end, we design a distributed learning algorithm where each agent $i$ processes the output of each other agent $j\neq i$ using dedicated coordination layers $\Lambda_{j}^{i}$. The loss function for agent $i$ is then given by:
\begin{multline}
L_{i}\left(\boldsymbol{M},\boldsymbol{\Lambda}\right)=\mathbb{E}\left\{ \left\Vert Y_{i}-\left(M_{i}X_{i}+\sum_{j=1,j\neq i}^{N}\Lambda_{j}^{i}M_{j}X_{j}\right)\right\Vert ^{2}\right\} =\\
\mathbb{E}\left\{ \left\Vert \left(A_{i}-M_{i}\right)X_{i}+\sum_{j=1,j\neq i}^{N}\left(B^{i}_{j}-\Lambda_{j}^{i}M_{j}\right)X_{j} + \nu_i \right\Vert ^{2}\right\} . \label{eq:6}
\end{multline}
Therefore, each agent's loss depends on the models of other agents. 

Our distributed training approach introduces a continuous game with the agents as players. The action of each player $i$ is its local model $M_i$ and $\Lambda_{j}^{i}$ for each $j\neq i$. At each turn $t$, each agent will choose a new action  $M_i$ and $\{\Lambda_{j}^{i}\}_j$. The cost function of a player is just its loss function given in \eqref{eq:6}.  The game induced by \eqref{eq:6} presents a conflict between the players. Each player seeks to adjust $M_i$ to minimize their own loss function. However, by doing so they change the loss values of others. In general, pushing $M_i$ in the direction of $A_i^{*}$ (defined in Assumption~1 below) conflicts with pushing $M_i$ such that $\Lambda_i^jM_i = B_i^j$.

Motivated by practical settings, we make the following assumption, which means that the agent model matrices are not arbitrary. Instead, their right processing (as modeled by $M_i$) can ``filter" the input for all agents. For example, while one agent may only be interested in whether the digit the other sees is even, the first layers that are required to detect the digit from the background are helpful for all agents. 

\begin{assum}
Define $A_i^{*}=A_i + \sum_{j\neq i}B^{i}_{j}\Cov\left(X_{j},X_{i}\right)$ where $\Cov\left(X_{j},X_{i}\right)$ is the covariance matrix of $X_j$ and $X_i$. We assume there exists ${\Lambda_{j}^{i*}}$ such that $\Lambda_{j}^{i*}A^{*}_j=B^{i}_{j}$ for all $i\neq j$. \label{assumption1}
\end{assum}

$A_i^{*}$ is the optimal linear estimator of $Y_i$ from $X_i$. Therefore, the assumption above means that the information that agent $i$ can potentially extract from its input, given an optimal model, is enough to compute the effect of agent $i$ on all other agents.

Assumption 1 is much weaker than assuming that \textit{every} improvement of $M_i$ towards $A_i^{*}$ also improves the quality of the information received by other agents, i.e., pushes closer to $\Lambda_i^jM_i = B_i^j$. Such an assumption would have made the game induced by \eqref{eq:6} a potential game, with the sum of loss functions as the potential function \citep{heliou2017learning}. Nevertheless, the cordial learning game is not a potential game. Often, updates of $M_i$ by agent $i$ to decrease its loss function in \eqref{eq:6} increase the loss values for other agents.

Although cordial learning can work without Assumption \ref{assumption1}, its optimality is based on this assumption. Naturally, the more reasonable this assumption is in a given scenario, the more valuable the information the agent sends to its peers. 

Our goal is to design a distributed training algorithm that minimizes the global objective:
\begin{equation}
    \underset{\left\{ M_{i}\right\} _{i},\left\{ \Lambda_{j}^{i}\right\} _{i\neq j}}{\min}\sum_{i=1}^{N}\mathbb{E}\left\{ \left\Vert Y_{i}-\left(M_{i}X_{i}+\sum_{j=1,j\neq i}^{N}\Lambda_{j}^{i}M_{j}X_{j}\right)\right\Vert ^{2}\right\}. \label{eq:7}
\end{equation}

In the CL algorithm, agent $i$ uses the local variables $Z_i^i=X_i$ and $Z_j^i=X_j-\operatorname{Cov}(X_j,X_i)X_i$ for $j\ne i$. Thus, $Z_i^i$ is uncorrelated with $Z_j^i$ for every $j\ne i$. We then work with the residualized versions of
\eqref{eq:6} and \eqref{eq:7}, obtained by replacing $X_j$
with $Z_j^i$ in the loss function of agent $i$.

Assumption~\ref{assumption1} implies that $M_i=A_i^*$ and
$\Lambda_j^i=\Lambda_j^{i*}$ for all $i$ and $j\ne i$
constitute a globally optimal solution of \eqref{eq:7} when using the $Z_j^i$ variables, achieving the minimal total MSE $N\sigma^2$.
However, the sum of the agents' loss functions is nonconvex
(see Appendix~\ref{sec:Nonconvex Objective}).
Each agent's loss is only convex in its own parameters when
the other agents' embedding matrices are held fixed.
Thus, convergence to a global minimum is not guaranteed
even for centralized gradient descent over
$\{M_i,\Lambda_j^i\}_{i,j}$.
This constitutes a major technical challenge in designing
a distributed algorithm that optimizes the residualized
version of \eqref{eq:7}.

\section{Cordial Learning Algorithm and Performance}
We propose a novel algorithm, termed cordial learning (correlated and distributed learning), to deal with the correlated data scenario. In cordial learning (CL), agents only share their embedded outputs and never their explicit datapoints. Then, each agent's model includes coordination layers to process the inputs from other agents. The pseudocode for CL is given in Algorithm 1. To keep it tractable, our theoretical model is linear \citep{maurer2006bounds}, but CL applies more broadly to any scenario in which training data is correlated between agents. We support this claim with experiments on deep neural networks (see Section 6).

CL is distributed since each agent calculates the gradients over its local parameters only. Instead of optimizing the global objective \eqref{eq:7} directly, each agent optimizes its local objective given in \eqref{eq:6}. The local loss function depends on $M_i,\Lambda_{1}^{i},\ldots,\Lambda_{N}^{i}$ for agent $i$, where $M_j$ for $j\neq i$, the models of other agents, are treated as constants.

Centralized gradient descent that optimizes \eqref{eq:7} would require each agent $i$ to know the gradient of the loss function of agent $j$ with respect to $M_i$. CL bypasses that, as each agent runs a two-time-scale SGD to minimize its loss function (\ref{eq:6}) with respect to its parameters alone. The stochastic gradients of an agent employ information sent by others. 

In general, SGD does not converge in games. When it does converge, it is to a stationary point that, under mild conditions, is a
Nash equilibrium (NE) \citep{mazumdar2020gradient}, which is typically socially inefficient \citep{roughgarden2015intrinsic,sessa2019bounding}. Surprisingly, our main theorem shows that in the game induced by the designed loss functions in $\eqref{eq:6}$, a two-time-scale SGD converges to a global minimum of the sum of loss functions. 

\begin{algorithm}[h!]
\caption{\label{alg:Sharing-information-agents}Cordial Learning}

\textbf{Initialization: }Let $Y_i$ be the $m\times 1$ label vector of agent $i$. Let $X_i$ be the $d\times 1$ input vector of agent $i$. $M_i, \{\Lambda_j^i\}_j$ are initialized with random values for each $i$. Let $\eta(t),\varepsilon(t),\rho(t)$ be sequences that are tuned according to the conditions of Theorem \ref{thm:main}.

\textbf{For $t=0,\hdots,T_{iter}-1$ training iterations:}
\begin{enumerate}

\item \textbf{Sample Batches:}
Sample two independent batches $B_t^0,B_t^1$ of $B$ data indices uniformly at random from a global index set shared across agents.
Stack the local input vectors corresponding to these indices in $X_i^{B_t^0}$ and $X_i^{B_t^1}$, matrices of dimensions $d\times B$.

\item \textbf{Communication:}
Send the embedded outputs $M_iX_i^{B_t^0}$ and $M_iX_i^{B_t^1}$ to all other agents.

\item \textbf{Construct Local Variables:}
For each sample ID $l$ in $B_t^0$ and each $j\neq i$, let $\tilde{R}_{j,0}^{i,l}=\frac{1}{B-1}\sum_{k\neq l, k \in B_t^0}X_j^kX_i^{kT}$ and $\tilde{R}_{j,1}^{i}
=
\frac{1}{B}\sum_{k\in B_t^1}X_j^kX_i^{kT}$. Compute
\begin{equation}
M_j\tilde{Z}_{j,0}^{l}
=
M_jX_j^l-M_j\tilde{R}_{j,0}^{i,l}X_i^l,
\qquad
M_j\tilde{Z}_{j,1}^{l}
=
M_jX_j^l-M_j\tilde{R}_{j,1}^{i}X_i^l.
\label{eq:8}
\end{equation}

\item \textbf{Update $M_i$:}

\begin{equation}
M_i(t+1)
=
M_i(t)-\varepsilon(t)\boldsymbol{g}_{M_i}(t)
\label{eq:9}
\end{equation}
where
\begin{multline}
\boldsymbol{g}_{M_i}(t)
=
-2{Y_i}^{B_t^0}(X_i^{B_t^0})^T
+2M_i(t)X_i^{B_t^0}(X_i^{B_t^0})^T
+\\
2\sum_{k=1,k\neq i}^{N}
\left(
\Pi_t\left[\Lambda_k^i(t)\right]
M_k(t)\tilde{Z}_{k,0}^{B_t^0}(t)
\right)
(X_i^{B_t^0})^T
\label{eq:10}
\end{multline}
and $\Pi_t[\Lambda]$ is the projection into the set
$\left\{
\Lambda\in\mathbb{R}^{m\times m}
\mid
\|\Lambda\|_F \le \rho(t)
\right\}$.

\item \textbf{Update $\Lambda_j^i$ for all $j\neq i$:}

\begin{equation}
\Lambda_j^i(t+1)
=
\Lambda_j^i(t)-\eta(t)\boldsymbol{g}_{\Lambda_j^i}(t)
\label{eq:11}
\end{equation}
where
\begin{multline}
\boldsymbol{g}_{\Lambda_j^i}(t)
=
2\left(
M_i(t)X_i^{B_t^0}-{Y_i}^{B_t^0}
\right)
\left(
M_j(t)\tilde{Z}_{j,1}^{B_t^0}(t)
\right)^T
+\\
2\sum_{k\neq i}
\Lambda_k^i(t)
M_k(t)\tilde{Z}_{k,0}^{B_t^0}(t)
\left(
M_j(t)\tilde{Z}_{j,1}^{B_t^0}(t)
\right)^T .
\label{eq:12}
\end{multline}

\end{enumerate}

\textbf{End}
\end{algorithm}

\begin{figure*}[h]
\centering
    {%
    \includegraphics[width=1\linewidth,keepaspectratio]{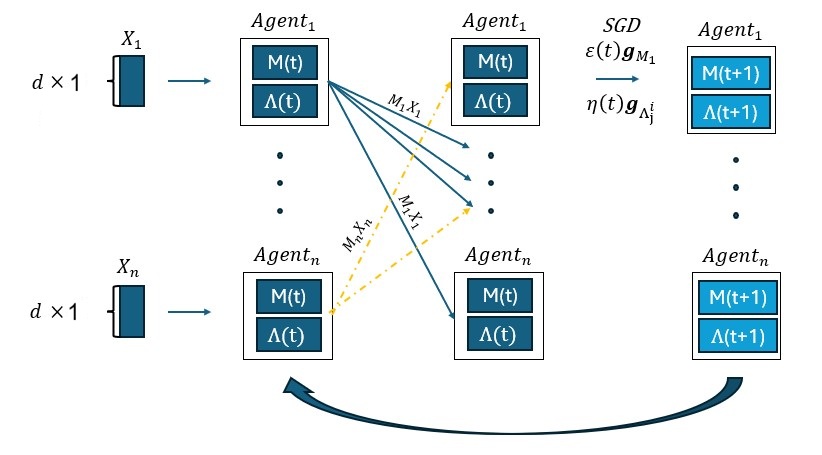}%
    \caption{Cordial Learning information flow}
    \label{fig:algo}%
    }
\end{figure*}

To deal with the correlated input data, each agent $i$ subtracts the component that is correlated with $X_i$ from all of its received data, resulting in local data variables $\tilde{Z}_1^{i},\ldots,\tilde{Z}_N^{i}$. The goal of these local variables is to make the expected gradient with respect to $M_i$ independent of $\{\Lambda^i_j\}_{j\neq i}$ and of $\{M_j\}_{j\neq i}$ (Lemma~3), and also reduce the variance of the stochastic gradients (Lemma \ref{lem:variance} and Lemma \ref{lem: schur}).

Let $\Cov\left(X_{j},X_{i}\right)$ be the covariance matrix of $X_j$ and $X_i$. Define $Z_{i}^{i}=X_{i}$ and $Z_{j}^{i}=X_{j}-\Cov\left(X_{j},X_{i}\right)X_{i}$, which is the version of $X_{j}$ after subtracting the information it holds about $X_{i}$. We then have 

\begin{equation}
    M_{j}Z_{j}^{i}=M_{j}X_{j}-M_{j}\Cov\left(X_{j},X_{i}\right)X_{i}=M_{j}X_{j}-\Cov\left(M_{j}X_{j},X_{i}\right)X_{i}\label{eq:13}
\end{equation}
which implies that agent $i$, which receives $M_jX_j$, can compute $M_{j}Z_{j}^{i}$ if it knows $\Cov\left(M_{j}X_{j},X_{i}\right)$. 

To run CL, $M_jZ_j^i$ is needed in computing the stochastic
gradients $\boldsymbol{g}_{M_i}$ and
$\boldsymbol{g}_{\Lambda_j^i}$. Since
$\operatorname{Cov}(X_j,X_i)$ is unknown, CL constructs two
independent estimates of this covariance. At each iteration, two
independent batches $B_t^0$ and $B_t^1$ are sampled. The batch
$B_t^0$ is used to compute the stochastic gradient, while $B_t^1$
is used to construct an independent covariance estimate.

For each sample $l$ in $B_t^0$, define
\[
\tilde{R}_{j,0}^{i,l}
=
\frac{1}{B-1}
\sum_{k\neq l, k\in B_t^{0}}X_j^kX_i^{kT},
\qquad
\tilde{R}_{j,1}^{i}
=
\frac{1}{B}
\sum_{k\in B_t^1}X_j^kX_i^{kT}.
\]
The
corresponding local variables are
\[
\tilde{Z}_{j,0}^{i,l}
=
X_j^l-\tilde{R}_{j,0}^{i,l}X_i^l,
\qquad
\tilde{Z}_{j,1}^{i,l}
=
X_j^l-\tilde{R}_{j,1}^{i}X_i^l,
\]
where for $j = i$, we set $\tilde{Z}^{i,l}_{i,0} = X^l_i$.

The covariance estimators are independent of the $l$-th sample,
and independent of each other. CL uses
$\tilde{Z}_{j,0}^{i,l}$ in the stochastic gradient with respect to
$M_i$. In the stochastic gradient with respect to
$\Lambda_j^i$, it uses $\tilde{Z}_{k,0}^{i,l}$ in the prediction
term and the independent $\tilde{Z}_{j,1}^{i,l}$ in the
multiplicative term. As Lemma~\ref{lem:expectation} shows, the resulting stochastic gradients are unbiased. 

CL utilizes common randomness between the agents, which generate the correlated samples that correspond to the same sample ID. This can be implemented using traditional epochs where agents go over all datapoints in a pre-agreed order. 

The communication between agents in CL can be implemented with or without a central server. The agents communicate both in training and online after deployment, to extract information from the correlation between their online test (unlabeled) samples using the coordination layers they trained offline. At test time, agent $i$ computes $M_j Z^i_j = M_j X_j - \mathrm{Cov}(M_j X_j, X_i) X_i$ using an estimate of $\mathrm{Cov}(M_j X_j, X_i)$ obtained from training samples with the final trained models.

Our main theorem states that CL converges with probability 1 to a global minimum of the sum of loss functions of all agents. 

Theorem \ref{thm:main} makes a standard SGD assumption on the step sizes of $M_i(t)$ and $\Lambda_{j}^{i}(t)$. On top of that, we assume that the update of the $\Lambda_{j}^{i}\left(t\right)$ is slower than that of the $M_{i}\left(t\right)$, formalized by the requirement that $\sum_{t=1}^{\infty}\frac{\eta^{2}\left(t\right)}{\varepsilon\left(t\right)}<\infty$. 

The projection of $\Lambda_{k}^{i}\left(t\right)$ into the $\rho\left(t\right)$ radius ball is only done when calculating the stochastic gradient $\boldsymbol{g}_{M_i}(t) $. The coordination layers $\Lambda_{k}^{i}\left(t\right)$ only affect the variance of  $\boldsymbol{g}_{M_i}(t) $ and not its bias. Hence, this projection ensures that the variance of $\boldsymbol{g}_{M_i}(t) $ does not explode. This is guaranteed by the assumption that $\sum_{t=1}^{\infty}\varepsilon^{2}\left(t\right)\rho^{2}\left(t\right)<\infty$.

\begin{thm}
    \label{thm:main} Consider $N$ agents that run cordial learning (Algorithm 1), where the $2B$ samples of $[X_n, Y_n]_{n=1}^N$ in $B_t^0$ and $B_t^1$ are independent within each iteration and IID over $t$. Assume that  $\{A_i \}_i$ and $
    \{B^{i}_{j}\}_{i,j}$ satisfy Assumption 1. Let $\varepsilon(t),\eta(t)$ be step size sequences of the embedding layers and coordination layers, respectively, and $\rho\left(t\right)$ be the projection radius. Assume that
  $\sum_{t=1}^{\infty}\varepsilon\left(t\right)=\infty$,
    $\sum_{t=1}^{\infty}\eta\left(t\right)=\infty$, $\rho\left(t\right) \rightarrow \infty$ and $\sum_{t=1}^{\infty}\varepsilon^{2}\left(t\right)\rho^{2}\left(t\right)<\infty$. Further assume that $\sum_{t=1}^{\infty}\frac{\eta^{2}\left(t\right)}{\varepsilon\left(t\right)}<\infty$. Then, for all $i$, the loss function
    \begin{equation}       L_{i}\left(\boldsymbol{M}(t),\boldsymbol{\Lambda}(t)\right)=\mathbb{E}\left\{ \left\Vert Y_{i} - M_{i}(t)Z_{i}^{i} - \sum_{k=1,k\neq i}^{N}\Lambda_{k}^{i}(t)M_{k}(t)Z_{k}^{i}\right\Vert ^{2}\right\}
    \end{equation}
    converges to its global minimum with probability 1. 
\end{thm}

\section{Convergence Analysis}
In this section, we detail our convergence analysis that leads to the performance guarantees of Theorem \ref{thm:main}. Throughout our analysis, we use the following filtration, which summarizes the past iterates of the algorithm:

\begin{equation}
\mathcal{F}_t=\sigma\left(\left\{ M_{i}\left(t\right)\right\} _{i=1}^{N},\left\{ \Lambda_{j}^{i}\left(t\right)\right\} _{i\neq j},\ldots,\left\{ M_{i}\left(0\right)\right\} _{i=1}^{N},\left\{ \Lambda_{j}^{i}\left(0\right)\right\} _{i \neq j}\right).
\end{equation}

Lemma \ref{lem:expectation} below computes the stochastic gradients used in Algorithm 1.  Naturally, the stochastic gradient of the loss with respect to the embedding parameters $M_i$ depends on the models of other agents, and on the coordination parameters $\boldsymbol{\Lambda}^i$. However, the expected gradients of the loss function of agent $i$ with respect to $M_i$ depend only on $M_i$. This unique fact is crucial to our analysis. In particular, we can project $\Lambda_{j}^{i}\left(t\right)$ into a ball with a controlled radius $\rho\left(t\right)$ without creating any bias in $\boldsymbol{g}_{M_{i}}$. Consequently, we can first show the convergence of $M_i\left(t\right)$, which is then used to prove convergence of $\Lambda_{j}^{i}\left(t\right)$.

\begin{lem} \label{lem:expectation}
    Let $C_{k,j}$ be the covariance matrix between $Z^{i}_k$ and $Z^{i}_j$. Recall that in Algorithm~1, agent $i$ uses, for each sample in the batch, the stochastic gradients
    \begin{equation}
        \boldsymbol{g}_{M_{i}}
        =-2Y_{i}\tilde{Z}_{i,0}^{iT}+2M_{i}\tilde{Z}_{i,0}^{i}\tilde{Z}_{i,0}^{iT}+2\left(\sum_{j=1,j\neq i}^{N}\Pi_t\left[\Lambda_{j}^{i}\left(t\right)\right]M_{j}\tilde{Z}_{j,0}^{i}\right)\tilde{Z}_{i,0}^{iT}\label{eq:15}
    \end{equation}
    and for an agent $j\neq i$
    \begin{equation}        \boldsymbol{g}_{\Lambda_{j}^{i}}
        = -2Y_{i}(M_{j}\tilde{Z}_{j,1}^{i})^{T}+2M_{i}\tilde{Z}_{i,0}^{i}(M_{j}\tilde{Z}_{j,1}^{i})^{T}+2\left(\sum_{k=1,k\neq i}^{N}\Lambda_{k}^{i}M_{k}\tilde{Z}_{k,0}^{i}\right)(M_{j}\tilde{Z}_{j,1}^{i})^{T}.\label{eq:16}
    \end{equation}
   Then, the expectations of the stochastic gradients are given by
    \begin{equation}
    \mathbb{E}\{ \boldsymbol{g}_{M_{i}}\}= 2\left(M_{i}-A_{i}-\sum_{j=1,j\neq i}^{N}B^{i}_{j}\Cov\left(X_{j},X_{i}\right)\right) \label{eq:17}
\end{equation}
and
\begin{equation}
     \mathbb{E}\{\boldsymbol{g}_{\Lambda_{j}^{i}}\}= 2\sum_{k=1,k\neq i}^{N}\left(\Lambda_{k}^{i}M_{k}-B^{i}_{k}\right)C_{k,j}M_{j}^{T}.\label{eq:18}
\end{equation}
\end{lem}

In particular, the fact that the expected gradient with respect to $M_i$ is $2(M_i-A_i^{*})$ suggests that in the absence of noise, $M_i$ will converge to $A_i^{*}$ with no consideration to other agents. This would not be the case if each agent took a gradient step to minimize the global objective in \eqref{eq:7}. The gradient of \eqref{eq:7} with respect to $M_i$ has $N-1$ other terms that push $M_i$ to solve $\Lambda_i^jM_i = B^{j}_{i}$.

The next lemma bounds the variances of our stochastic gradients. We see that $\Lambda_j^{i}$ converges to minimal error only since $M_j$ converges to $A_j^{*}$ for other agents, which indicates that learning is coupled between agents. Nevertheless, the convergence of $\Lambda_j^{i}$ also helps learning $M_i$, since the variance of the stochastic gradient of $M_i$ tends to decrease as the error of $\Lambda_j^{i}$ decreases (see \eqref{eq:21} in the proof of Lemma \ref{lem:variance}). This shows that the learning of $M_i$ and $\Lambda_j^{i}$ is coupled. The decreasing variance of $\boldsymbol{g}_{M_i}$ as $\Lambda_j^{i}$ converges is the reason for working with the local $\tilde{Z}_1^{i},\ldots,\tilde{Z}_N^{i}$ variables, which subtract the influence of $X_i$ from the received variables.

The bounds of Lemma \ref{lem:variance} are worst case and therefore do not show the dependence on the convergence error of $\Lambda_j^{i}$. They also ignore the dependence on the variance of $\tilde{R}_{j}^{i}$ that decreases in the batch size. The proof of Theorem \ref{thm:main} only needs the bounds as stated in Lemma \ref{lem:variance}. Performance in practice, of course, benefits a lot from the decreasing variances of the stochastic gradients. 

\begin{lem}
\label{lem:variance}Let $\Lambda_{j}^{i*}$ be such that $B^{i}_{j}=\Lambda_{j}^{i*}A_{j}^{*}$.
Recall the stochastic gradients $\boldsymbol{g}_{M_{i}}\left(t\right)$ and $\boldsymbol{g}_{\Lambda_{j}^{i}}\left(t\right)$ defined in \eqref{eq:15} and \eqref{eq:16}, for the parameters at iteration $t$. Then for all $i$
\begin{equation}
\mathbb{E}\{\left\Vert \boldsymbol{g}_{M_{i}}\left(t\right)\right\Vert ^{2}|\mathcal{F}_{t}\}=O\left(\rho^{2}\left(t\right)\left(1+\sum_{j=1}^{N}\left\Vert M_{j}\left(t\right)-A_{j}^{*}\right\Vert ^{2}\right)\right)\label{eq:19}
\end{equation}
and for all $i,j$, understanding that $\Lambda_{i}^{i}=0$,
\begin{equation}
\mathbb{E}\left\{ \left\Vert \boldsymbol{g}_{\Lambda_{j}^{i}}\left(t\right)\right\Vert ^{2}|\mathcal{F}_{t}\right\} =O\left(\sum_{l=1}^{N}\left(1+\left\Vert \Lambda_{l}^{i}\left(t\right)\right\Vert ^{2}\right)\sum_{k=1}^{N}\left(1+\left\Vert M_{k}\left(t\right)-A_{k}^{*}\right\Vert ^{2}\right)^{2}\right).\label{eq:20}
\end{equation}
\end{lem}
\begin{proof}
Let $\tilde{R}_{j}^{i}=\frac{1}{B-1}\sum_{k\neq l}{X}_{j}^{k}{X}_{i}^{kT}$ be the estimated covariance matrix based on the batch samples. Let $\tilde{A}_{i}=A_{i}+\sum_{j\neq i}B^{i}_{j}\tilde{R}_{j}^{i}$, which approximates $A_i^{*}$. For sufficiently large constants $K_{0},\ldots,K_{10}>0$
\begin{multline}
\frac{1}{4}\mathbb{E}\{\left\Vert \boldsymbol{g}_{M_{i}}\left(t\right)\right\Vert ^{2}|\mathcal{F}_{t}\}=\\ \mathbb{E}\Big\{\Big\Vert\left(A_{i}-M_{i}\left(t\right)\right)X_{i}X_{i}^{T}+\sum_{j\neq i}\left(B_{j}^{i}X_{j}-\Pi_t\left[\Lambda_{j}^{i}\left(t\right)\right]M_{j}\left(t\right)\tilde{Z}_{j,0}^{i}\right)X_{i}^{T}+\nu_iX_{i}^{T}\Big\Vert^{2}\ \Big|\ \mathcal{F}_{t}\Big\}\underset{\left(a\right)}{=} \\
\mathbb{E}\Big\{\Big\Vert\left(A_{i}-M_{i}\left(t\right)\right)X_{i}X_{i}^{T}+\sum_{j\neq i}\left(B_{j}^{i}\left(\tilde{Z}_{j,0}^{i}+\tilde{R}_{j}^{i}X_{i}\right)-\Pi_t\left[\Lambda_{j}^{i}\left(t\right)\right]M_{j}\left(t\right)\tilde{Z}_{j,0}^{i}\right)X_{i}^{T}+\nu_iX_{i}^{T}\Big\Vert^{2}\ \Big|\ \mathcal{F}_{t}\Big\}= \\
\mathbb{E}\Big\{\Big\Vert\left(A_{i}+\sum_{j\neq i}B_{j}^{i}\tilde{R}_{j}^{i}-M_{i}\left(t\right)\right)X_{i}X_{i}^{T}+\sum_{j\neq i}\left(B_{j}^{i}-\Pi_t\left[\Lambda_{j}^{i}\left(t\right)\right]M_{j}\left(t\right)\right)\tilde{Z}_{j,0}^{i}X_{i}^{T}+\nu_iX_{i}^{T}\Big\Vert^{2}\ \Big|\ \mathcal{F}_{t}\Big\}\leq \\
3\mathbb{E}\left\{ \left\Vert \nu_iX_{i}^{T}\right\Vert ^{2}\right\} +3\mathbb{E}\left\{ \left\Vert M_{i}\left(t\right)-\tilde{A}_{i}\right\Vert ^{2}\left\Vert X_{i}X_{i}^{T}\right\Vert ^{2}\vert\mathcal{F}_{t}\right\} \\ +3N\mathbb{E}\left\{ \sum_{j\neq i}\left\Vert B_{j}^{i}-\Pi_t\left[\Lambda_{j}^{i}\left(t\right)\right]M_{j}\left(t\right)\right\Vert ^{2}\left\Vert \tilde{Z}_{j,0}^{i}X_{i}^{T}\right\Vert ^{2}\vert\mathcal{F}_{t}\right\} \underset{\left(b\right)}{\leq}  \\
3\mathbb{E}\left\{ \left\Vert \nu_i\right\Vert ^{2}\right\} \mathbb{E}\left\{ \left\Vert X_{i}\right\Vert ^{2}\right\} +3\mathbb{E}\left\{ \left\Vert M_{i}\left(t\right)-\tilde{A}_{i}\right\Vert ^{2}\right\} \mathbb{E}\left\{ \left\Vert X_{i}X_{i}^{T}\right\Vert ^{2}\right\} \\ +3N\sum_{j\neq i}\left\Vert B_{j}^{i}-\Pi_t\left[\Lambda_{j}^{i}\left(t\right)\right]M_{j}\left(t\right)\right\Vert ^{2}\mathbb{E}\left\{ \left\Vert \tilde{Z}_{j,0}^{i}X_{i}^{T}\right\Vert ^{2}\right\} \underset{\left(c\right)}{\leq} \\
K_{0}\sigma^{2}+K_{6}\mathbb{E}\left\{ \left\Vert M_{i}\left(t\right)-\tilde{A}_{i}\right\Vert ^{2}\right\} +K_{1}\sum_{j\neq i}\left(\left\Vert B_{j}^{i}-\Pi_t\left[\Lambda_{j}^{i}\left(t\right)\right]A_{j}^{*}\right\Vert ^{2}+\left\Vert \Pi_t\left[\Lambda_{j}^{i}\left(t\right)\right]\left(A_{j}^{*}-M_{j}\left(t\right)\right)\right\Vert ^{2}\right)\leq \\
K_{0}\sigma^{2}+2K_{6}\left\Vert M_{i}\left(t\right)-A_{i}^{*}\right\Vert ^{2}+2K_{6}\mathbb{E}\left\{ \left\Vert \tilde{A}_{i}-A_{i}^{*}\right\Vert ^{2}\right\} \\ +K_{1}\sum_{j\neq i}\left(\left\Vert B_{j}^{i}-\Pi_t\left[\Lambda_{j}^{i}\left(t\right)\right]A_{j}^{*}\right\Vert ^{2}+\left\Vert \Pi_t\left[\Lambda_{j}^{i}\left(t\right)\right]\right\Vert ^{2}\left\Vert A_{j}^{*}-M_{j}\left(t\right)\right\Vert ^{2}\right)\label{eq:21}
\end{multline}
where (a) uses that $\tilde{Z}_{j,0}^{i}=X_{j}-\tilde{R}_{j}^{i}X_{i}$, (b) uses that $\nu_i$ is independent of $X_i$ and that $\tilde{R}^i_j$ is independent of the $l$-th sample, and (c) uses that $\mathbb{E}\left\{ \left\Vert X_{i}X_{i}^{T}\right\Vert ^{2}\right\} $
and $\mathbb{E}\left\{ \left\Vert \tilde{Z}_{j,0}^{i}X_{i}^{T}\right\Vert ^{2}\right\} $
are bounded.

Similarly, for the stochastic gradient with respect to $\Lambda_j^{i}$, 
\begin{multline}
\frac{1}{4}\mathbb{E}\left\{ \left\Vert \boldsymbol{g}_{\Lambda_{j}^{i}}\left(t\right)\right\Vert ^{2}|\mathcal{F}_{t}\right\} =\\ \mathbb{E}\left\{ \left\Vert \left(\left(M_{i}\left(t\right)-A_{i}\right)X_{i}+\sum_{k\neq i}\left(\Lambda_{k}^{i}\left(t\right)M_{k}\left(t\right)\tilde{Z}_{k,0}^{i}-B_{k}^{i}X_{k}\right)-\nu_i\right)\tilde{Z}_{j,1}^{iT}M_{j}^{T}\left(t\right)\right\Vert ^{2}|\mathcal{F}_{t}\right\} \leq \\
3\mathbb{E}\left\{ \left\Vert \nu_i\tilde{Z}_{j,1}^{iT}M_{j}^{T}\left(t\right)\right\Vert ^{2}\right\} +3\mathbb{E}\left\{ \left\Vert \left(M_{i}\left(t\right)-\tilde{A}_{i}\right)X_{i}\tilde{Z}_{j,1}^{iT}M_{j}^{T}\left(t\right)\right\Vert ^{2}|\mathcal{F}_{t}\right\} 
\\ +3N\sum_{k\neq i}\mathbb{E}\left\{ \left\Vert \left(\Lambda_{k}^{i}\left(t\right)M_{k}\left(t\right)-B_{k}^{i}\right)\tilde{Z}_{k,0}^{i}\tilde{Z}_{j,1}^{iT}M_{j}^{T}\left(t\right)\right\Vert ^{2}|\mathcal{F}_{t}\right\} \leq \\
3\mathbb{E}\left\{ \left\Vert \nu_i\tilde{Z}_{j,1}^{iT}M_{j}^{T}\left(t\right)\right\Vert ^{2}\right\} +3\mathbb{E}\left\{ \left\Vert M_{i}\left(t\right)-\tilde{A}_{i}\right\Vert ^{2}\right\} \left\Vert M_{j}\left(t\right)\right\Vert ^{2}\mathbb{E}\left\{ \left\Vert X_{i}\tilde{Z}_{j,1}^{iT}\right\Vert ^{2}\right\} 
\\ +3N\sum_{k\neq i}\left\Vert \Lambda_{k}^{i}\left(t\right)M_{k}\left(t\right)-B_{k}^{i}\right\Vert ^{2}\left\Vert M_{j}\left(t\right)\right\Vert ^{2}\mathbb{E}\left\{ \left\Vert \tilde{Z}_{k,0}^{i}\tilde{Z}_{j,1}^{iT}\right\Vert ^{2}\right\} \underset{\left(a\right)}{\leq} \\
K_{0}\sigma^{2}+K_{2}\left(1+\left\Vert M_{i}\left(t\right)-A_{i}^{*}\right\Vert ^{2}+\mathbb{E}\left\{ \left\Vert \tilde{A}_{i}-A_{i}^{*}\right\Vert ^{2}\right\} \right)\left(\left\Vert M_{j}\left(t\right)-A_{j}^{*}\right\Vert ^{2}+\left\Vert A_{j}^{*}\right\Vert ^{2}\right)
\\ +NK_{3}\sum_{k\neq i}\left(\left\Vert \Lambda_{k}^{i}\left(t\right)A_{k}^{*}-B_{k}^{i}\right\Vert ^{2}+\left\Vert \Lambda_{k}^{i}\left(t\right)\right\Vert ^{2}\left\Vert M_{k}\left(t\right)-A_{k}^{*}\right\Vert ^{2}\right)\left(\left\Vert M_{j}\left(t\right)-A_{j}^{*}\right\Vert ^{2}+\left\Vert A_{j}^{*}\right\Vert ^{2}\right)\leq \\
K_{5}\left(\left\Vert M_{i}\left(t\right)-A_{i}^{*}\right\Vert ^{2}+\left\Vert M_{i}\left(t\right)-A_{i}^{*}\right\Vert ^{2}\left\Vert M_{j}\left(t\right)-A_{j}^{*}\right\Vert ^{2}+\left\Vert M_{j}\left(t\right)-A_{j}^{*}\right\Vert ^{2}+\sigma^{2}+1\right)\\ +
K_{6}\sum_{k\neq i}\left\Vert \Lambda_{k}^{i}\left(t\right)A_{k}^{*}-B_{k}^{i}\right\Vert ^{2}+K_{7}\left\Vert M_{j}\left(t\right)-A_{j}^{*}\right\Vert ^{2}\sum_{k\neq i}\left\Vert \Lambda_{k}^{i}\left(t\right)A_{k}^{*}-B_{k}^{i}\right\Vert ^{2}\\ +
K_{8}\left\Vert M_{j}\left(t\right)-A_{j}^{*}\right\Vert ^{2}\sum_{k\neq i}\left\Vert \Lambda_{k}^{i}\left(t\right)\right\Vert ^{2}\left\Vert M_{k}\left(t\right)-A_{k}^{*}\right\Vert ^{2}
+K_{9}\sum_{k\neq i}\left\Vert \Lambda_{k}^{i}\left(t\right)\right\Vert ^{2}\left\Vert M_{k}\left(t\right)-A_{k}^{*}\right\Vert ^{2}\leq \\ K_{10}\sum_{l=1}^{N}\left(1+\left\Vert \Lambda_{l}^{i}\left(t\right)\right\Vert ^{2}\right)\sum_{k=1}^{N}\left(1+\left\Vert M_{k}\left(t\right)-A_{k}^{*}\right\Vert ^{2}\right)^{2}\label{eq:27}
\end{multline}
where in (a) we used that $\mathbb{E}\left\{ \left\Vert \nu_i\tilde{Z}_{j,1}^{iT}\right\Vert ^{2}\right\} ,\mathbb{E}\left\{ \left\Vert \tilde{Z}_{k,0}^{i}\tilde{Z}_{j,1}^{iT}\right\Vert ^{2}\right\} ,\mathbb{E}\left\{ \left\Vert X_{i}\tilde{Z}_{j,1}^{iT}\right\Vert ^{2}\right\}$
are bounded and 
\begin{equation}
    3\mathbb{E}\left\{ \left\Vert \nu_i\tilde{Z}_{j,1}^{iT}M_{j}^{T}\left(t\right)\right\Vert ^{2}\right\} \leq 6\mathbb{E}\left\{ \left\Vert \nu_i\tilde{Z}_{j,1}^{iT}A_{j}^{*T}\right\Vert ^{2}\right\} +6\mathbb{E}\left\{ \left\Vert \nu_i\tilde{Z}_{j,1}^{iT}\right\Vert ^{2}\left\Vert M_{j}^{T}\left(t\right)-A_{j}^{*T}\right\Vert ^{2}\right\}. 
\end{equation}
\end{proof}

The following lemma is not necessary to prove Theorem \ref{thm:main}. However, it sheds more light on how CL works. While Theorem \ref{thm:main} shows that the loss of each agent goes to its global minimum, Lemma \ref{lem: schur} shows that
each effective coordination error $\Lambda_j^i(t)A_j^*-B_j^i$
converges to zero with probability one. Since $\Lambda^i_j(t)$ is bounded and $\rho(t) \to \infty$, we have $\Pi_t\left[\Lambda^i_j(t)\right] = \Lambda^i_j(t)$ for all large $t$. Together with $M_j(t) \to A^*_j$ (Theorem~\ref{thm:main}) and the bound \eqref{eq:21} in the proof of Lemma~\ref{lem:variance}, it implies that the variance of $\boldsymbol{g}_{M_i}(t)$ has a decreasing upper bound that reaches an error floor, determined by the label noise and covariance estimation.

The assumptions of Lemma~\ref{lem: schur} essentially rule out
redundant dimensions in the learning problem. If a nonzero
linear combination of $X_1,\ldots,X_N$ is constant, then some
feature observed by one agent is perfectly determined by features
observed by the others and can therefore be removed without losing
information.

\begin{lem} \label{lem: schur}
Assume that no nonzero linear combination of $X_1,\ldots,X_N$ is a constant with probability 1. Then in CL, for each $i,j$ ($i\neq j$), $\ensuremath{\left\Vert \ensuremath{\Lambda_{j}^{i}\left(t\right)A_{j}^{*}-B_{j}^{i}}\right\Vert } \rightarrow 0$ with probability 1 as $t\rightarrow \infty$.
\end{lem}
\begin{proof}
Let $f_{i}\left(\Lambda_{1}^{i},\ldots,\Lambda_{N}^{i}\right)\triangleq\mathbb{E}\left\{ \left\Vert \sum_{k=1,k\neq i}^{N}\left(\Lambda_{k}^{i}A_{k}^{*}-B_{k}^{i}\right)Z_{k}^{i}\right\Vert ^{2}\right\}$.  Without loss of generality, let $i=1$, and define
\[
    \Delta(t)
    =
    \begin{pmatrix}
        \Lambda_2^1(t)A_2^*-B_2^1
        & \cdots &
        \Lambda_N^1(t)A_N^*-B_N^1
    \end{pmatrix}.
\]
Let $C_1$ be the block covariance matrix of
$(Z_2^1,\ldots,Z_N^1)$, whose $(k,j)$ block is $C_{k,j}$.
Then

\begin{equation}
    f_1\!\left(
        \Lambda_1^1(t),\ldots,\Lambda_N^1(t)
    \right)
    =
    \mathbb{E}\left\|
        \sum_{k=2}^N
        \left(
            \Lambda_k^1(t)A_k^*-B_k^1
        \right)Z_k^1
    \right\|^2 \\
    =
    \operatorname{Tr}\left(
        \Delta(t)C_1\Delta^T(t)
    \right).
\end{equation}

Let $R_{j,k}$ be the covariance matrix of $X_j$ and $X_k$. Define 
\begin{equation}
R_1=\left(\begin{array}{c|c}
\underbrace{R_{1,1}}_{D} & \underbrace{\begin{pmatrix} R_{1,2} & \cdots & R_{1,N} \end{pmatrix}}_{B} \\
\hline
\underbrace{\begin{pmatrix} R_{2,1} \\ \vdots \\ R_{N,1} \end{pmatrix}}_{B^{T}} & \underbrace{ \begin{pmatrix} R_{2,2} & \cdots & R_{2,N} \\ \vdots & \ddots & \vdots \\ R_{N,2} & \cdots & R_{N,N} \end{pmatrix} }_{U}
\end{array}\right). \label{eq:26a}
\end{equation}

The Schur complement of $D$ in $R_1$
is $R_1 / D=U-B^{T}D^{-1}B$, which equals $C_1$ since $D=I$. Let $X$ be the vector that stacks $X_1,\ldots,X_N$. If $R_1$ is not full-rank, then there exists $u\neq\boldsymbol{0}$ such that $R_1u=0$. However, the variance of $u^{T}X$ is $u^{T}R_1u=0$, so $u^{T}X$ is a constant with probability 1, contradicting the assumption. Therefore, $R_1$ is positive definite, implying that its Schur complement $C_1$ is also positive definite. 

Therefore,
\[
\begin{aligned}
    f_1\!\left(
        \Lambda_1^1(t),\ldots,\Lambda_N^1(t)
    \right)
    &=
    \operatorname{Tr}\left(
        \Delta(t)C_1\Delta^T(t)
    \right) \\
    &\geq
    \lambda_{\min}(C_1)\|\Delta(t)\|_F^2 
    =
    \lambda_{\min}(C_1)
    \sum_{j=2}^N
    \left\|
        \Lambda_j^1(t)A_j^*-B_j^1
    \right\|_F^2.
\end{aligned}
\]
The proof of Theorem~2 shows that $f_1\!\left(
        \Lambda_1^1(t),\ldots,\Lambda_N^1(t)
    \right)
    \rightarrow 0$ with probability one. Since $\lambda_{\min}(C_1)>0$, it follows that
\[
    \left\|
        \Lambda_j^1(t)A_j^*-B_j^1
    \right\|_F
    \longrightarrow 0
\]
for every $j\ne1$. The same argument applies to every agent $i$.

\end{proof}

\newsavebox{\secsixbox}
\newcommand{\secsixfit}[1]{\sbox{\secsixbox}{#1}%
  \ifdim\wd\secsixbox>\linewidth\resizebox{\linewidth}{!}{\usebox{\secsixbox}}\else\usebox{\secsixbox}\fi}
  
\section{Experiments}
\label{sec:experiments}

In this section, we test CL numerically. The first set of simulations
uses our linear theoretical model, assumed in Theorem~2. The second set goes beyond the linear
model and runs CL with deep neural networks that classify the digits of a composite
MNIST image.

\subsection{Linear Model Simulations}
\label{sec:exp-linear}

\paragraph{Input Data Generation.}
We simulated $N \in \{8, 16, 32\}$ agents. Agent $i$ has an input $X_i \in \mathbb{R}^d$ and a
label $Y_i \in \mathbb{R}^m$, with $d = 2$ and $m = 1$. As assumed in Section~3.1, the inputs have zero mean and are whitened, so that $\mathbb{E}\{X_i X_i^T\} = I$. The inputs of different agents
are correlated coordinate by coordinate. For each coordinate $k = 1, \ldots, d$, the vector
$(X_{1,k}, \ldots, X_{N,k})$ has a correlation matrix $R \in \mathbb{R}^{N \times N}$, and different
coordinates are independent, so that $\mathrm{Cov}(X_j, X_i) = R_{ji} I$. In every run we draw a new
matrix $R$ with unit diagonal and off-diagonal entries $(U_{ji} + U_{ij})/2$, where the entries of
$U$ are drawn uniformly from $[-5/N, 5/N]$. Such a matrix is not always positive semidefinite. In
that case, we replace it by a nearby correlation matrix: we raise its eigenvalues to at least
$10^{-12}$ and rescale it to a unit diagonal. This was needed in about 60\% of the runs with $N = 8$,
in less than 10\% of the runs with $N = 16$, and in none with $N = 32$.

We use two input distributions. In the first, the inputs are Gaussian with correlation matrix $R$.
In the second, the inputs have Laplace marginals: we draw Gaussian inputs as before and map every
coordinate $g$ to $F^{-1}(\Phi(g))$, where $\Phi$ is the standard normal distribution function and
$F$ is that of the Laplace distribution with unit variance. The coordinates of each agent stay
independent with unit variance, but their correlations across agents change. For Laplacian inputs, $\mathrm{Cov}(X_j, X_i) = \tilde{R}_{ji} I$, where $\tilde{R}_{ji}$ is the correlation of the transformed coordinates, which we compute from $R_{ji}$ by numerical integration. In the rest of this section, $\mathrm{Cov}(X_j, X_i)$ refers to this matrix for Laplacian inputs.

\paragraph{Output Generation.}
The entries of $A_i \in \mathbb{R}^{m \times d}$ and $\Lambda^{i*}_j \in \mathbb{R}^{m \times m}$
are drawn independently from $\mathcal{N}(0, 0.5^2)$ for Gaussian inputs and from a Laplace
distribution with scale $0.4$ for Laplacian inputs. To ensure that the fixed-point iteration below
is a contraction, the coordination matrices of each agent $i$ are scaled by
\[
  s_i = \min\Bigl\{1, \frac{\kappa}{\gamma_i}\Bigr\}, \qquad
  \gamma_i = \sum_{j \neq i} \bigl\|\Lambda^{i*}_j\bigr\|_2 \bigl\|\mathrm{Cov}(X_j, X_i)\bigr\|_2,
\]
with $\kappa = 0.8$, where $\|\cdot\|_2$ is the spectral norm. We then compute $A^*_i$ as the fixed point of the Picard
iteration
\[
  A^*_i \leftarrow A_i + \sum_{j \neq i} \Lambda^{i*}_j A^*_j \mathrm{Cov}(X_j, X_i),
\]
set $B^i_j = \Lambda^{i*}_j A^*_j$, and generate the labels as
$Y_i = A_i X_i + \sum_{j \neq i} B^i_j X_j$. By construction, Assumption~1 holds and $A^*_i$ is
the optimal linear estimator of $Y_i$ from $X_i$. The labels are noiseless ($\sigma^2 = 0$), so the
global minimum of the loss~(14) of every agent is $0$.

\paragraph{Training Procedure.}
The agents run Algorithm~1. At each iteration, they sample two independent batches
${B}_t^0$ and ${B}_t^1$ of $B = 10$ fresh samples, so the samples are independent
across iterations, as Theorem~2 assumes. The stochastic gradients (10) and (12) are averaged over
the batch. The initial values $M_i(0)$ and $\Lambda^i_j(0)$ are random, with entries drawn from the
same distribution as the entries of $A_i$ and of $\Lambda^{i*}_j$ before scaling. The step
sizes and the projection radius are
\[
  \varepsilon(t) = \frac{L_m}{(t + t_0)^{\alpha}}, \qquad
  \eta(t) = \frac{L_\lambda}{(t + t_0)^{\beta}}, \qquad
  \rho(t) = (t + 1)^{0.05},
\]
with $\alpha = 0.6$, $\beta = 1$ and $t_0 = 100$, and with $L_m$ and $L_\lambda$ such that
$\varepsilon(0) = 0.1$ and $\eta(0) = 0.05$. These sequences satisfy all the conditions of
Theorem~2: $\sum_t \varepsilon(t) = \sum_t \eta(t) = \infty$, $\rho(t) \to \infty$,
$\sum_t \varepsilon^2(t) \rho^2(t) < \infty$ and $\sum_t \eta^2(t)/\varepsilon(t) < \infty$.
Isolated agents minimize~(4) by SGD with the same step size $\varepsilon(t)$, using the batch
${B}_t^0$. Each agent trains for 20{,}000 iterations.

\paragraph{Evaluation.}
For cordial agents, we report the loss~(14) of each agent, which is the objective of Theorem~2,
averaged over the $N$ agents. For isolated agents, we report the loss~(4). The smallest possible
isolated loss is $\mathbb{E}\{\|Y_i - A^*_i X_i\|^2\}$, attained by $M_i = A^*_i$. We compute
both losses in closed form from the known input covariance, so the curves carry no evaluation noise.
The residualization in~(14) uses the true covariance $\mathrm{Cov}(X_j, X_i)$, while the agents
themselves only use the estimates of Algorithm~1. Each configuration was repeated for 100
independent runs, each with new parameters, a new initialization and new samples.

\paragraph{Results.}
Figure~\ref{fig:linear} shows the loss of cordial and isolated agents for all six configurations. Isolated agents converge to
the best isolated predictor within about 100 iterations. Their median loss is at most
$1.6\%$ above the best isolated loss from then on. Cordial agents start with a higher loss,
since their coordination matrices are initialized at random, but it drops quickly. After 20
iterations, it is lower than that of the isolated agents in 94--100\% of the runs of each
configuration, and from iteration 140 on, it is lower in all 600 runs. The loss of
cordial agents then keeps decreasing towards the global minimum. After 20{,}000 iterations, the
median ratio between the loss of cordial and isolated agents ranges from $0.36\%$ ($N = 8$,
Gaussian inputs) to $1.4\%$ ($N = 32$, Laplacian inputs), and it is below $2.4\%$ in
every run. On the log-log scale, the loss of cordial agents decreases almost linearly, with a slope
between $-0.58$ and $-0.49$ from iteration 2{,}000 to 20{,}000 in all configurations. The
loss therefore decreases like a power of $t$. The loss at the last iteration is larger with more agents and with Laplacian inputs, while the rate of decrease is
about the same in all configurations.

\begin{figure}[t]
  \centering
  \includegraphics[width=\textwidth]{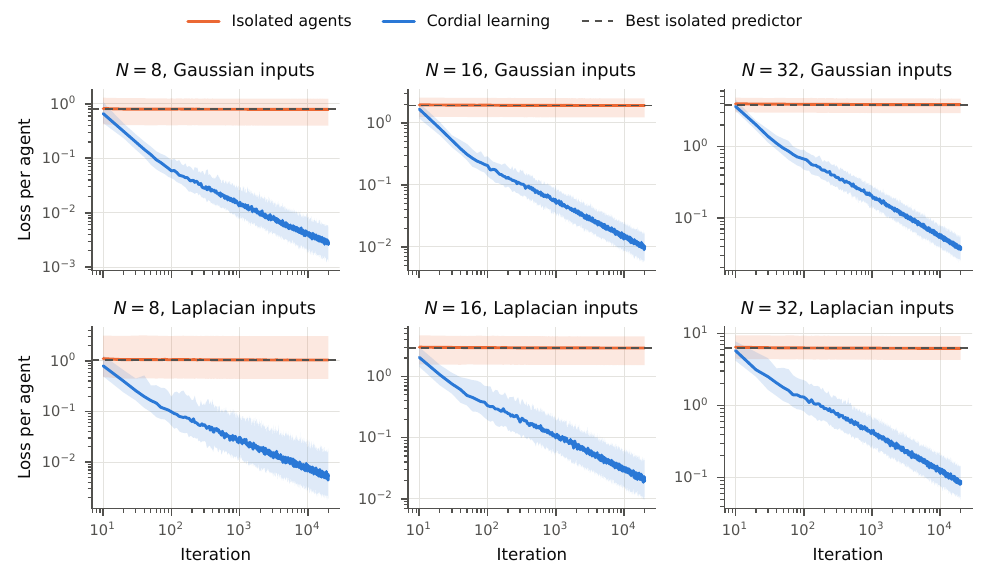}
  \caption{Linear model: loss of cordial agents~(14) and of isolated agents~(4) for
  $N = 8, 16, 32$ agents with Gaussian (top) and Laplacian (bottom) inputs. The lines show the median
  over 100 runs of the loss averaged over the agents, and the shaded bands show the 10th to 90th
  percentiles across runs. The dashed line is the median loss of the best isolated predictor
  $M_i = A^*_i$. The global minimum of the loss of cordial agents is $0$. Both axes are
  logarithmic.}
  \label{fig:linear}
\end{figure}

\subsection{MNIST-Based Multi-Agent Simulations}
\label{sec:exp-mnist}

\paragraph{Input Data Generation.}
Each data sample is a horizontal concatenation of $N$ MNIST digits, which forms a $28 \times 28N$
image. Agent $i$ sees the $i$-th $28 \times 28$ digit and has to classify it. The digits of a sample
are correlated, and the rule that generates them defines the dataset. Section~\ref{sec:exp-pairwise}
uses a deterministic rule (Pairwise-10), and Section~\ref{sec:exp-common} a probabilistic one (a
common cause). We vary the amount of data by the number $P$ of MNIST images per digit.
The training composites are built from a pool of $P$ MNIST training images of each digit, and the
validation composites from a second pool of $P$ training images of each digit, disjoint from the
first. The test composites are built from a pool of $P$ MNIST test images of each digit. Therefore,
no image appears in two of the three sets. A composite takes a random image from the pool of each of
its digits. Pixel values are scaled to $[-1, 1]$. We call one such random construction of the training, validation and test sets a draw of the images.

\paragraph{Neural Network Architectures.}
Every agent runs the LeNet-5-style network of Figure~\ref{fig:Architecture-shared}. Its embedding layers
are \texttt{conv1} (6 filters of size $5 \times 5$, padded to keep the $28 \times 28$ size),
\texttt{pool1} ($2 \times 2$ max pooling), \texttt{conv2} (16 filters of size $5 \times 5$),
\texttt{pool2}, \texttt{conv3} (120 filters of size $5 \times 5$), \texttt{fc1} (84 units) and
\texttt{fc2} (10 units). All layers use ReLU activations, and dropout follows every convolution,
every pooling layer and \texttt{fc1}. The output $h_i \in \mathbb{R}^{10}$ of \texttt{fc2}, after
its ReLU, is the embedded output of agent $i$, the analogue of $M_i X_i$. We compare three architectures:
\begin{itemize}
  \item \textbf{Cordial Learning}: Each agent sends $h_i$ to all others, both in training and
  at test time. During training, the communicated embedded outputs are computed with dropout, as in the forward pass. In its layer \texttt{fc3}, agent $i$ computes
  $z_i = W_i h_i + b_i + \sum_{j \neq i} \Lambda^i_j h_j + c_i$, where
  $W_i, \Lambda^i_j \in \mathbb{R}^{10 \times 10}$ are weight matrices and $b_i, c_i \in \mathbb{R}^{10}$ are biases, and its
  prediction is the softmax of \texttt{fc4}$(\mathrm{ReLU}(z_i))$, where \texttt{fc4} has 10 units.
  The matrices $\Lambda^i_j$ and the bias $c_i$ form the coordination layer, which adds the message term
  $\sum_{j \neq i} \Lambda^i_j h_j + c_i$ to $z_i$. It is initialized to zero. The received embedded outputs $h_j$ are treated as
  constants, so the parameters of agent $i$ are updated only by the gradient of its own loss, as in
  Algorithm~1.
  \item \textbf{Isolated Agents}: The same network without the message term. The agents do not
  communicate.
  \item \textbf{Centralized Architecture}: A single standard CNN that receives the entire
  $28 \times 28N$ image and predicts all the digits jointly. It has three convolutional layers with $5 \times 5$ filters (padded to keep the size), each followed by ReLU and $2 \times 2$ max pooling, a fully connected hidden
  layer, and an output layer with $10N$ units, with a softmax over each group of 10. Dropout follows
  every pooling layer and the hidden layer. The numbers of filters and hidden units are tuned from (6, 16, 120) filters with 128 hidden units, (16, 32, 64) with 128 or 256, and (32, 64, 128) with 256. It is trained on the sum of the losses of all
  agents.
\end{itemize}

\paragraph{Training and Evaluation.}
All models are trained with Adam on batches of 65 composites, reshuffled every epoch. The loss is
the sum over the agents of the cross-entropy of each agent's digit. In epoch $e$, the learning rate of all
layers except the coordination layers is $\varepsilon_e = \varepsilon_0 / (1 + e/t_0)^{0.6}$. The
coordination layers are not trained in the first $K$ epochs, and after that their learning rate is
$\eta_e = \eta_0 / (1 + (e - K)/t_0)$. Therefore, $\eta_e/\varepsilon_e \to 0$ and
$\sum_e \eta_e^2/\varepsilon_e < \infty$, as in Theorem~2. With Adam, however, these are nominal
learning rates and not the step sizes of Theorem~2. The delay $K$ lets the embeddings become informative
before the agents start to rely on them. Isolated agents and the centralized network use the
schedule $\varepsilon_e$ for all their layers.

The hyperparameters of every method were tuned with 4 agents on the validation set of the first draw
of the images (see below), separately for every dataset and data size, with the same number of runs
for all methods. The initial learning rate $\varepsilon_0$ is drawn log-uniformly from
$[5 \cdot 10^{-4}, 5 \cdot 10^{-2}]$ ($[10^{-4}, 10^{-2}]$ for the centralized network), $t_0$ from
$\{1, 3, 10, 30, 100, 300\}$ (without the largest values for larger $P$), and the dropout rate from
$\{0.1, 0.25, 0.5\}$. For cordial agents, $\eta_0$ is also drawn log-uniformly from
$[10^{-4}, 10^{-1}]$, and $K$ from a grid that scales with the number of epochs (e.g.,
$\{0, 25, 50, 100, 200, 400\}$ for $P = 25$). For the centralized network, we also tune the numbers
of filters and hidden units. For each method, we trained 40 random configurations, then 40
perturbations of the five best, and then retrained the three best with three more random seeds. For Pairwise-10 and for the common-cause dataset with $p = 0.6$, with $P = 25$ and $P = 500$, the candidates also include the configurations of an earlier random search, the same for all methods (12 and 10 per method, respectively, with epoch limits of 1{,}000 and 120). The configuration with the lowest validation loss, averaged over its four seeds, was selected. Tuning runs stop after
200, 80, and 15 epochs without an improvement of the validation loss for $P = 25$, 100, and
500, respectively, or at an epoch limit of 1{,}500, 600, and 180 epochs.

With the selected hyperparameters, each method was trained with 10 random seeds for a fixed number
of epochs: either the maximal number of epochs for tuning, or 1.5 times this number if one of the selected configurations reached its best validation loss in the last 10\% of the tuning epochs. The number of agents $N \in \{2, 4, 6\}$ changes, but the hyperparameters tuned with 4 agents are kept. We keep the network with the lowest validation loss during training, which
we evaluate every 5, 2, and 1 epochs for $P = 25, 100$, and 500 (for cordial agents, from
epoch $K$ on), and evaluate it on the test set without dropout. We report the test loss per agent
(the sum over agents of the mean test cross-entropy, divided by $N$) and the test accuracy averaged
over the agents. For cordial agents, we also report the test accuracy when the message term is
removed, which shows how much the agents rely on the received embedded outputs. The selected
hyperparameters are listed in Table~\ref{tab:hparams}.

With few images, the results depend on which images are drawn into the three pools. We therefore repeated the final runs of all methods on two further independent draws
of the training, validation and test pools, with the same hyperparameters, numbers of epochs and
random seeds, and report the results of all three draws together (30 runs per method). To compare two
methods, we compare every run of one method with every run of the other on the same draw (100
pairs per draw) and report the number of pairs in which the first method has the lower test loss.

 The received embedded outputs are not residualized, so the coordination layers use the raw values
  $h_j$. Without residualization, the second batch ${B}_t^1$ is not needed. We also do not project the coordination layers.

\begin{table}[t]
  \centering
  \scriptsize
  \setlength{\tabcolsep}{3pt}
  \secsixfit{\begin{tabular}{llrrrrrrrrrrrr}
\toprule
 & & \multicolumn{5}{c}{Cordial learning} & \multicolumn{3}{c}{Isolated agents} & \multicolumn{4}{c}{Centralized} \\
\cmidrule(lr){3-7} \cmidrule(lr){8-10} \cmidrule(lr){11-14}
Dataset & $P$ & $\varepsilon_0$ & $\eta_0$ & $K$ & $t_0$ & Drop. & $\varepsilon_0$ & $t_0$ & Drop. & $\varepsilon_0$ & $t_0$ & Width & Drop. \\
\midrule
Pairwise-10 & 25 & 0.0023 & 0.0052 & 50 & 300 & 0.5 & 0.0012 & 300 & 0.5 & 0.006 & 300 & 16-32-64; 128 & 0.5 \\
 & 100 & 0.0025 & 0.0071 & 100 & 100 & 0.5 & 0.0015 & 100 & 0.25 & 0.0041 & 100 & 16-32-64; 256 & 0.5 \\
 & 500 & 0.0011 & 0.018 & 10 & 10 & 0.25 & 0.001 & 30 & 0.25 & 0.0017 & 30 & 32-64-128; 256 & 0.5 \\
\midrule
Common, $p = 0.8$ & 25 & 0.0041 & 0.00017 & 100 & 100 & 0.5 & 0.0018 & 300 & 0.5 & 0.002 & 30 & 32-64-128; 256 & 0.5 \\
 & 100 & 0.0065 & 0.00019 & 0 & 30 & 0.25 & 0.0017 & 100 & 0.25 & 0.0025 & 100 & 32-64-128; 256 & 0.5 \\
 & 500 & 0.001 & 0.0001 & 0 & 30 & 0.25 & 0.0052 & 10 & 0.25 & 0.0015 & 30 & 32-64-128; 256 & 0.5 \\
\midrule
Common, $p = 0.6$ & 25 & 0.004 & 0.0013 & 400 & 30 & 0.5 & 0.0017 & 300 & 0.5 & 0.00075 & 10 & 32-64-128; 256 & 0.5 \\
 & 100 & 0.0044 & 0.00031 & 10 & 100 & 0.25 & 0.0014 & 30 & 0.25 & 0.0032 & 100 & 16-32-64; 128 & 0.5 \\
 & 500 & 0.0014 & 0.00028 & 0 & 30 & 0.25 & 0.0019 & 10 & 0.25 & 0.0012 & 30 & 32-64-128; 256 & 0.5 \\
\bottomrule
\end{tabular}
 }
  \caption{Selected hyperparameters for Pairwise-10 and the common-cause datasets with $p = 0.8$ and
  $p = 0.6$ (tuned with 4 agents): initial learning rates $\varepsilon_0$ and
  $\eta_0$, decay constant $t_0$, start $K$ of the coordination layers and dropout rate. The width of the
  centralized network gives the numbers of filters of its three convolutional layers and the number
  of hidden units.}
  \label{tab:hparams}
\end{table}

\begin{figure*}[h]
    \centering
        {%
        \includegraphics[width=0.8\linewidth,keepaspectratio]{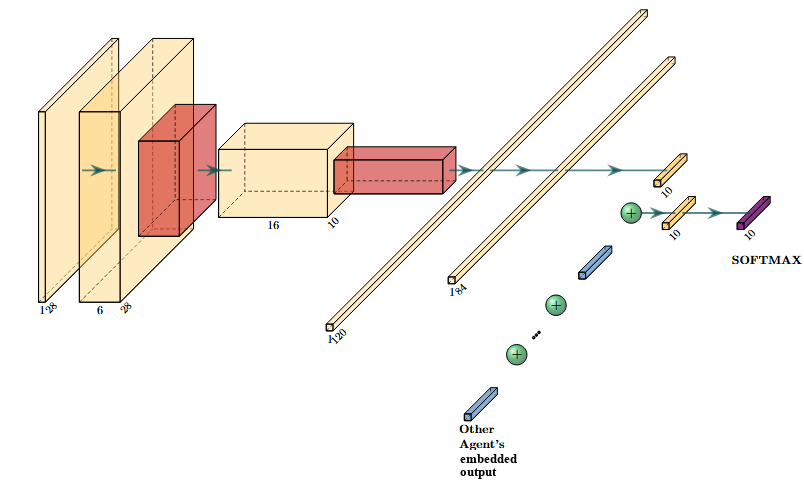}%
        \caption{NN Architecture}
        \label{fig:Architecture-shared}%
        }
\end{figure*}

\subsubsection{The Pairwise-10 Dataset}
\label{sec:exp-pairwise}

In the Pairwise-10 dataset, the digits of a sample satisfy two rules. The digits at positions
$(1,2), (3,4), \ldots$ sum to 10, and so do the digits at positions $(2,4), (4,6), \ldots$. The valid
tuples of digits are therefore $(a, 10-a)$, $(a, 10-a, 10-a, a)$ and $(a, 10-a, 10-a, a, a, 10-a)$ for
$N = 2, 4, 6$, with $a \in \{1, \ldots, 9\}$. In particular, the digit of any single agent determines
the digits of all other agents. Each of the three sets has $P$ composites of every valid tuple, that
is, $9P$ composites. We report $P = 25$, 100, and 500, from 225 to 4{,}500 composites per set. The final runs have 2{,}250, 600, and 180 epochs, respectively.

\paragraph{Results.}
Figure~\ref{fig:pairwise} shows the test loss during training for $P = 25$, 100 and 500, and
Table~\ref{tab:pairwise} lists the test loss and accuracy of the selected networks. CL is much better than isolated training in every setting. For $P = 25$,
the median test loss per agent of isolated agents is $0.45$, $0.30$ and $0.30$ for 2, 4 and 6 agents.
Cordial learning reduces it to $0.23$, $0.014$ and $0.0028$, and raises the median accuracy from
$0.88$, $0.91$ and $0.91$ to $0.94$, $0.996$ and $1.000$. For $P = 500$, isolated agents reach an
accuracy of $0.98$ and a test loss of $0.06$--$0.07$, and CL reduces the test loss to
$0.008$, $0.0009$ and less than $10^{-4}$. In every setting, every run of CL has a lower test loss than every run of isolated agents on the same draw, except for 2 agents with $P = 25$, where this holds for 297 of the 300 pairs of runs. The gain grows
with the number of agents, which is expected, since every additional agent brings another view of
the same information. The draw of the images matters for $P = 25$: the median test loss of isolated
agents with 2 agents is $0.31$, $0.47$ and $0.49$ on the three draws, and that of CL is
$0.093$, $0.30$ and $0.26$. The advantage of CL holds on every draw.

The centralized network shows what can be achieved when observing all digits. With $P = 500$, it reaches a test loss of  $0.002$ for 2 agents and less than
$0.0001$ for 4 and 6 agents. CL reaches almost the same accuracy (at least $0.998$),
although the agents only exchange their 10-dimensional embedded outputs and never their inputs or gradients. The centralized network has the lower test loss in most
pairs of runs.

With less data, the centralized network often failed to train. With the selected configuration, 47
of its 180 runs with $P = 25$ and $P = 100$ did not learn at all: the network predicts the uniform
distribution over the nine possible digits for every agent, with a test loss of $\ln 9 \approx 2.2$.
This happened in 7, 12 and 15 of the 30 runs with $P = 25$ and 2, 4 and 6 agents, and in 0, 6 and 7
of the 30 runs with $P = 100$. This occurred even though none of the four tuning runs of the selected configuration failed.
None of the 180 runs of CL with $P \le 100$ failed, and their lowest accuracy is $0.90$.
For $P = 25$ and 2, 4 and 6 agents, the median test loss of CL is $0.23$, $0.014$ and
$0.0028$, and that of the centralized network is $0.30$, $0.027$ and $1.1$. Without the runs that did
not learn, the median test loss of the centralized network is $0.23$, $0.0002$ and less than $10^{-4}$.

The gain of CL comes from the received embedded outputs. When the message term is
removed, the accuracy of cordial agents drops to between $0.65$ and $0.95$, below that of isolated
agents in every setting. Cordial agents learn to combine their own image with the embedded outputs
of the others, instead of classifying their own digit alone. This fits the structure of
Pairwise-10, where a single received embedded output can carry as much information about the digit
of agent $i$ as its own image.

\begin{figure}[t]
  \centering
  \includegraphics[width=\textwidth]{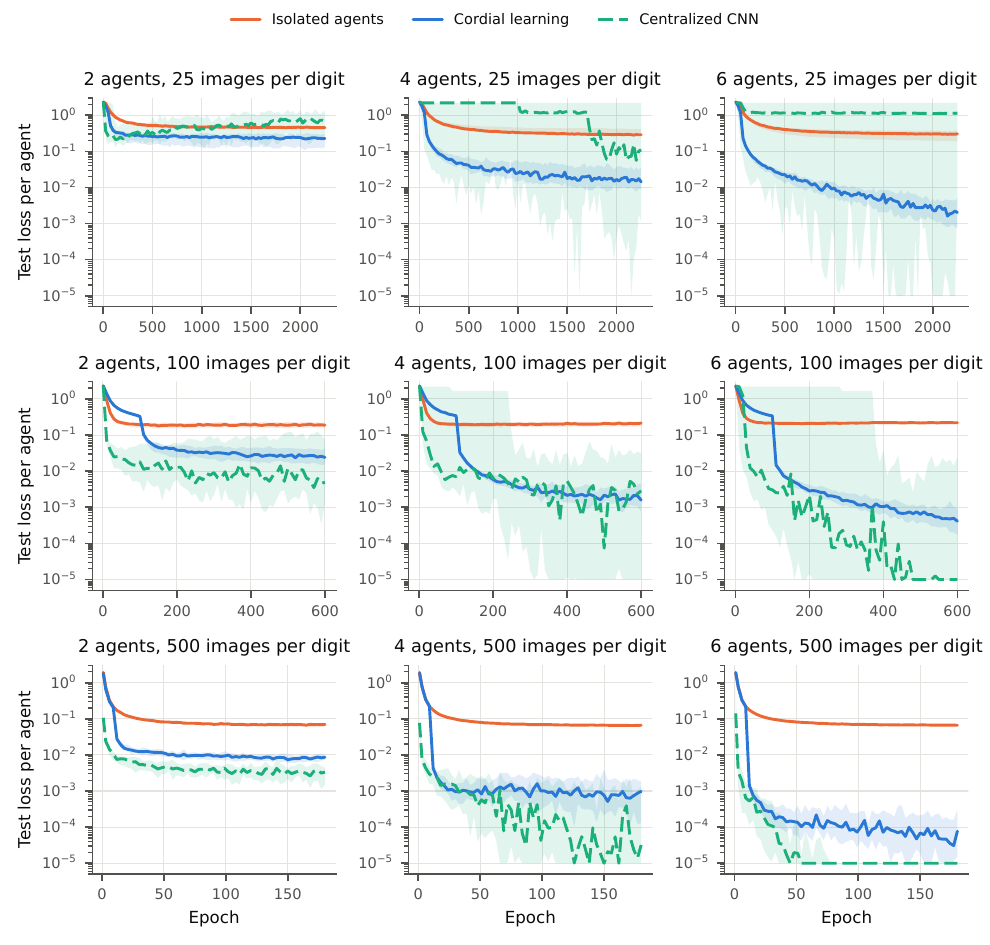}
  \caption{Pairwise-10: test loss per agent during training of isolated agents, cordial agents and
  the centralized network. The lines show the median over 30 runs (10 random seeds on each of three draws of the images) and
  the shaded bands the interquartile range. The sharp drop in the loss of cordial agents is at epoch
  $K$, when the coordination layers start training. The vertical axis is logarithmic, and losses below
  $10^{-5}$ (including losses that are exactly zero in single precision) are drawn at $10^{-5}$.}
  \label{fig:pairwise}
\end{figure}

\begin{table}[t]
  \centering
  \footnotesize
  \setlength{\tabcolsep}{4pt}
  \begin{tabular}{lllcccc}
\toprule
 & & & & \multicolumn{2}{c}{Test accuracy} & \\
\cmidrule(lr){5-6}
$P$ & $N$ & Method & Runs & & messages removed & Test loss per agent \\
\midrule
25 & 2 & Isolated & 30 & 0.883 & -- & 0.454 [0.397, 0.495] \\
 &  & Cordial learning & 30 & 0.944 & 0.760 & 0.232 [0.117, 0.302] \\
 &  & Centralized & 30 (7) & 0.963 & -- & 0.299 [0.146, 0.826] \\
\addlinespace
 & 4 & Isolated & 30 & 0.913 & -- & 0.295 [0.274, 0.421] \\
 &  & Cordial learning & 30 & 0.996 & 0.684 & 0.014 [0.0087, 0.036] \\
 &  & Centralized & 30 (12) & 0.996 & -- & 0.027 [0.00013, 2.197] \\
\addlinespace
 & 6 & Isolated & 30 & 0.911 & -- & 0.296 [0.191, 0.364] \\
 &  & Cordial learning & 30 & 1.000 & 0.651 & 0.0028 [0.00088, 0.0044] \\
 &  & Centralized & 30 (15) & 0.547 & -- & 1.128 [$<10^{-4}$, 2.197] \\
\midrule
100 & 2 & Isolated & 30 & 0.960 & -- & 0.182 [0.161, 0.205] \\
 &  & Cordial learning & 30 & 0.993 & 0.928 & 0.022 [0.016, 0.039] \\
 &  & Centralized & 30 & 0.998 & -- & 0.015 [0.0013, 0.071] \\
\addlinespace
 & 4 & Isolated & 30 & 0.961 & -- & 0.194 [0.159, 0.214] \\
 &  & Cordial learning & 30 & 0.999 & 0.903 & 0.0016 [0.00089, 0.0032] \\
 &  & Centralized & 30 (6) & 0.999 & -- & 0.0017 [$<10^{-4}$, 0.015] \\
\addlinespace
 & 6 & Isolated & 30 & 0.953 & -- & 0.208 [0.189, 0.216] \\
 &  & Cordial learning & 30 & 1.000 & 0.882 & 0.00032 [0.00022, 0.00068] \\
 &  & Centralized & 30 (7) & 1.000 & -- & $<10^{-4}$ [$<10^{-4}$, 0.023] \\
\midrule
500 & 2 & Isolated & 30 & 0.981 & -- & 0.069 [0.059, 0.073] \\
 &  & Cordial learning & 30 & 0.998 & 0.954 & 0.0078 [0.0066, 0.011] \\
 &  & Centralized & 30 & 0.999 & -- & 0.0024 [0.0017, 0.0036] \\
\addlinespace
 & 4 & Isolated & 30 & 0.982 & -- & 0.063 [0.062, 0.067] \\
 &  & Cordial learning & 30 & 1.000 & 0.910 & 0.00085 [0.00029, 0.0018] \\
 &  & Centralized & 30 & 1.000 & -- & $<10^{-4}$ [$<10^{-4}$, 0.0016] \\
\addlinespace
 & 6 & Isolated & 30 & 0.981 & -- & 0.066 [0.064, 0.069] \\
 &  & Cordial learning & 30 & 1.000 & 0.916 & $<10^{-4}$ [$<10^{-4}$, 0.00013] \\
 &  & Centralized & 30 & 1.000 & -- & $<10^{-4}$ [$<10^{-4}$, $<10^{-4}$] \\
\bottomrule
\end{tabular}
 
  \caption{Pairwise-10: test accuracy (averaged over the agents) and test loss per agent of the
  network with the lowest validation loss, median over the runs, with the interquartile range of the
  loss in brackets. \emph{Runs} is the number of runs: 30 for $P = 25$, 100 and 500 (three draws of the
  images, 10 random seeds each); the number of runs that did not learn (test loss
  $\ln 9$) is given in parentheses. For cordial agents, \emph{messages removed}
  is the test accuracy when the message term is removed.}
  \label{tab:pairwise}
\end{table}

\subsubsection{The Common-Cause Datasets}
\label{sec:exp-common}

In Pairwise-10, the digit of one agent determines the digits of all others. In the common-cause
datasets, the dependence is probabilistic and much weaker. A hidden class $z$ is drawn uniformly from
$\{0, \ldots, 9\}$, and the digit of each agent is independently equal to $z$ with probability $p$
and uniform on the nine other digits otherwise. The digit of every agent is therefore uniform, and
the digits of two agents are equal with probability $p^2 + (1-p)^2/9$. For $p = 0.1$ the digits are
independent, and as $p$ grows, the other agents carry more information about the digit of agent $i$.
We report $p = 0.8$ and $p = 0.6$. As a reference for this information, the best prediction of the
digit of an agent from the true digits of the other $N - 1$ agents alone has an accuracy of $0.64$,
$0.74$ and $0.79$ for $N = 2, 4, 6$ with $p = 0.8$, and $0.38$, $0.45$ and $0.54$ with $p = 0.6$. Each of
the three sets has $9P$ composites, the same number as in Pairwise-10. All hyperparameters were tuned separately for
each value of $p$ and $P$, as described above. The final runs have 1{,}500, 900, and 180 epochs
for $p = 0.8$ and 2{,}250, 900, and 180 epochs for $p = 0.6$, for $P = 25$, 100, and 500.

\paragraph{Results.}
Figures~\ref{fig:common80} and~\ref{fig:common60} show the test loss during training for
$p = 0.8$ and $p = 0.6$, and Table~\ref{tab:common} lists the test loss and accuracy of all methods.

For $p = 0.8$ and $P \ge 100$, CL is better than isolated training in every setting
and on every draw of the images. It has the lower test loss in 256 to 300 of the 300 pairs of runs
for $P = 100$ and $P = 500$. Its median test loss is
18\% to 35\% lower, and its error rate is 15\% to 28\% lower than that of isolated agents. The gain is
larger with 4 and 6 agents than with 2 agents. For $P = 500$ and 6 agents the test loss drops from $0.068$ to $0.044$ and the accuracy rises from $0.982$ to $0.987$. For
$P = 25$, CL is better with 6 agents (274 of 300 pairs, a 9\% lower loss); with 2
and 4 agents, the two methods are on par, where CL wins in 155 and 191 of 300 pairs, respectively. 

For $p = 0.6$, the less correlated scenario, the gain of CL depends on the amount of data. For $P = 500$, CL is better on every draw, in 246 to 282 of 300 pairs, with a 7\% to 13\% lower test loss. For $P = 100$, it is better with 4 agents (210 of 300 pairs) and on par with 2 and 6 agents. For $P = 25$, CL is worse than isolated training
with 4 agents on every draw (78 of 300 pairs) and with 6 agents on two of the three draws (98 of 300
pairs), and on par with 2 agents. A likely reason is that the coordination layers overfit. With
225 training composites, the embedded output of a neighbor may identify its digit almost perfectly on
the training images but not on new images, and the dependence between the digits is too weak to
make up for these errors. Tuning selected the latest start of the
coordination layers in its grid ($K = 400$ of 2{,}250 epochs), which suggests that on the
validation set, the received embedded outputs help little. 

When the message term is removed, the accuracy of cordial agents drops in every setting with
$P \ge 100$, and in 7 of these 12 settings it falls below that of isolated agents. For example, with $p=0.8$, accuracy drops from $0.985$ to $0.977$ for $P = 500$ and 4 agents, against $0.980$ for isolated agents. Unlike in Pairwise-10, where the received embedded outputs can replace the agent's own image, here they presumably correct the predictions on images whose digit is ambiguous.

For $P = 500$, the centralized network is the best method and its
test loss is 19\% to 36\% lower than that of CL. However, with $P = 25$, it is the worst method, with a test loss of $0.39$ to $0.97$, against $0.30$ to $0.38$ for isolated agents and $0.31$ to $0.42$ for CL. Its validation loss is
lowest after a median of 96 to 242 epochs, while CL reaches its lowest validation loss
near the end of training (after a median of 1{,}314 to 2{,}144 epochs), so the centralized network
presumably overfits the 225 training composites. For $P = 100$, 18 of its 90 runs with $p = 0.8$ and
2 of its 90 runs with $p = 0.6$ did not learn at all and stayed at the uniform prediction (test loss
$\ln 10$), as in Pairwise-10. 
CL has a lower median test loss than the centralized network in all settings with
$P = 25$, for $p = 0.8$ with $P = 100$ and 4 or 6 agents, and for $p = 0.6$ with $P = 100$ and 6 agents. Unlike in Pairwise-10, this is not due to the runs that did not learn: without them, the median
test loss of the centralized network for $P = 100$ is $0.16$ and $0.30$ for $p = 0.8$ with 4 and
6 agents, and $0.26$ for $p = 0.6$ with 6 agents.

\begin{figure}[t]
  \centering
  \includegraphics[width=\textwidth]{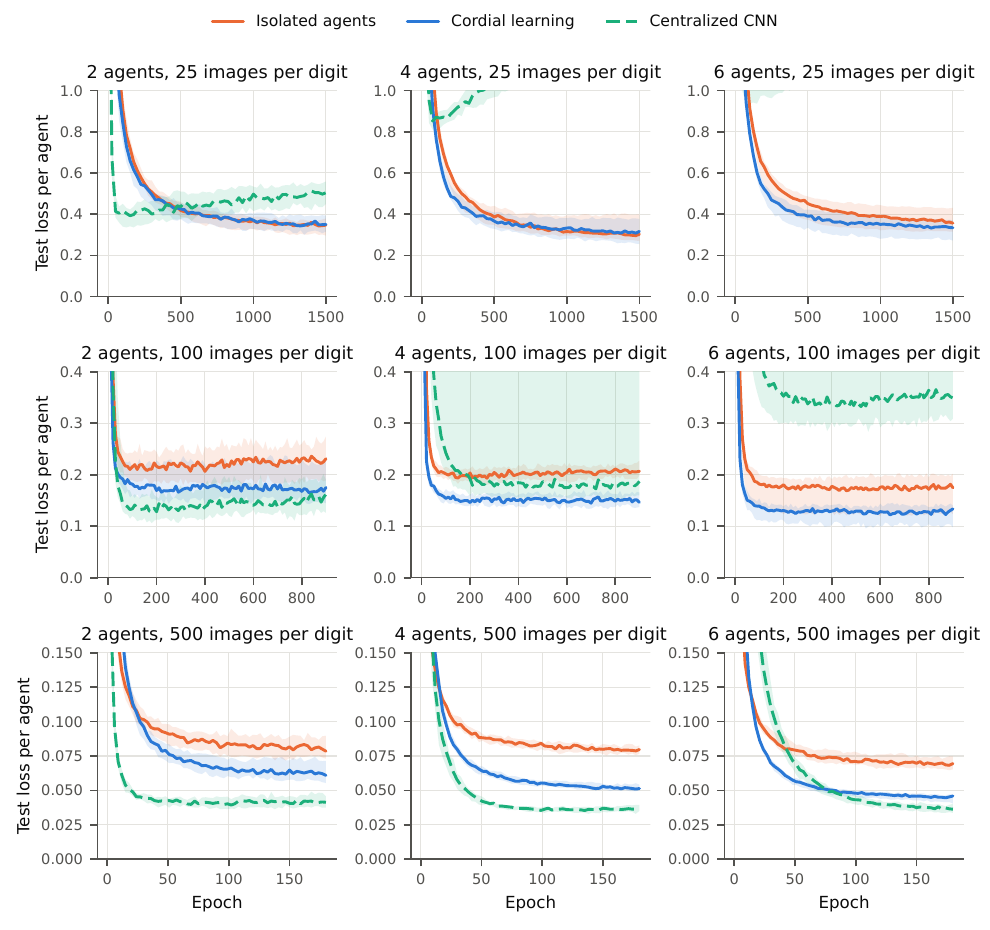}
  \caption{Common-cause dataset with $p = 0.8$: test loss per agent during training of isolated
  agents, cordial agents and the centralized network. The lines show the median over 30 runs (10 random seeds on each of 3
  draws) and the shaded bands the interquartile range.}
  \label{fig:common80}
\end{figure}

\begin{figure}[t]
  \centering
  \includegraphics[width=\textwidth]{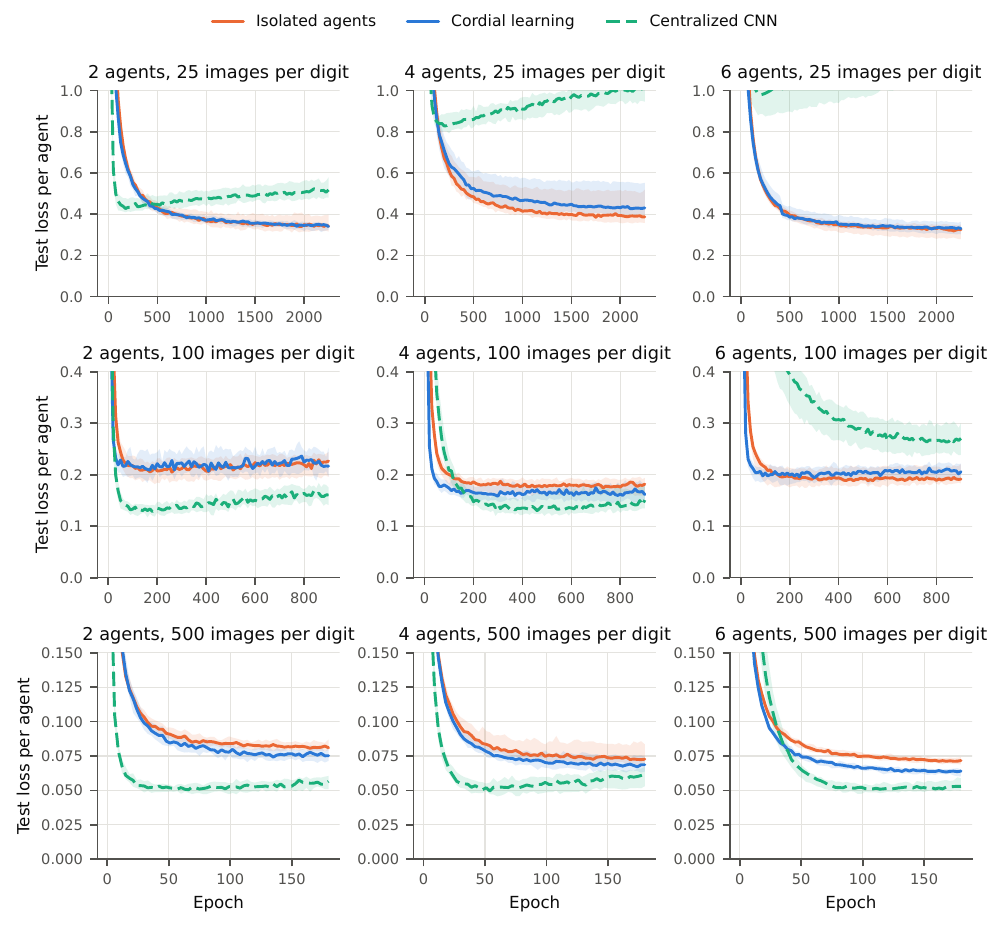}
  \caption{Common-cause dataset with $p = 0.6$: test loss per agent during training, as in
  Figure~\ref{fig:common80}.}
  \label{fig:common60}
\end{figure}

\begin{table}[t]
  \centering
  \scriptsize
  \setlength{\tabcolsep}{3.5pt}
  \begin{tabular}{llrrrrrrrrrr}
\toprule
 & & & & \multicolumn{3}{c}{Test loss per agent} & \multicolumn{4}{c}{Test accuracy} & \\
\cmidrule(lr){5-7} \cmidrule(lr){8-11}
$p$ & $P$ & $N$ & Runs & ISO & CL & Cen. & ISO & CL & CL, removed & Cen. & CL $<$ ISO \\
\midrule
0.8 & 25 & 2 & 30 & 0.350 & 0.350 & 0.395 & 0.910 & 0.910 & 0.897 & 0.900 & 52\% \\
 &  & 4 & 30 & 0.298 & 0.310 & 0.827 & 0.914 & 0.911 & 0.891 & 0.771 & 64\% \\
 &  & 6 & 30 & 0.360 & 0.328 & 0.952 & 0.904 & 0.906 & 0.862 & 0.762 & 91\% \\
\addlinespace
 & 100 & 2 & 30 & 0.203 & 0.166 & 0.122$^{(1)}$ & 0.952 & 0.961 & 0.955 & 0.971 & 85\% \\
 &  & 4 & 30 & 0.196 & 0.147 & 0.175$^{(8)}$ & 0.956 & 0.967 & 0.958 & 0.961 & 99\% \\
 &  & 6 & 30 & 0.172 & 0.119 & 0.319$^{(9)}$ & 0.961 & 0.971 & 0.957 & 0.913 & 99.7\% \\
\addlinespace
 & 500 & 2 & 30 & 0.080 & 0.061 & 0.039 & 0.980 & 0.983 & 0.978 & 0.989 & 97\% \\
 &  & 4 & 30 & 0.079 & 0.051 & 0.035 & 0.980 & 0.985 & 0.977 & 0.991 & 100\% \\
 &  & 6 & 30 & 0.068 & 0.044 & 0.036 & 0.982 & 0.987 & 0.978 & 0.990 & 100\% \\
\midrule
0.6 & 25 & 2 & 30 & 0.334 & 0.341 & 0.428 & 0.912 & 0.909 & 0.908 & 0.878 & 52\% \\
 &  & 4 & 30 & 0.381 & 0.422 & 0.826 & 0.892 & 0.876 & 0.878 & 0.759 & 26\% \\
 &  & 6 & 30 & 0.319 & 0.326 & 0.975 & 0.916 & 0.908 & 0.899 & 0.703 & 33\% \\
\addlinespace
 & 100 & 2 & 30 & 0.208 & 0.213 & 0.134 & 0.953 & 0.958 & 0.957 & 0.968 & 43\% \\
 &  & 4 & 30 & 0.180 & 0.158 & 0.130$^{(1)}$ & 0.959 & 0.965 & 0.963 & 0.971 & 70\% \\
 &  & 6 & 30 & 0.191 & 0.195 & 0.257$^{(1)}$ & 0.952 & 0.954 & 0.950 & 0.936 & 53\% \\
\addlinespace
 & 500 & 2 & 30 & 0.080 & 0.074 & 0.049 & 0.978 & 0.980 & 0.979 & 0.987 & 82\% \\
 &  & 4 & 30 & 0.072 & 0.067 & 0.051 & 0.979 & 0.982 & 0.979 & 0.987 & 93\% \\
 &  & 6 & 30 & 0.072 & 0.063 & 0.050 & 0.981 & 0.982 & 0.979 & 0.988 & 94\% \\
\bottomrule
\end{tabular}
  \caption{Common-cause datasets: median test loss per agent and median test accuracy (averaged over
  the agents) of isolated agents (ISO), cordial learning (CL) and the centralized network (Cen.).
  \emph{Runs} is the number of runs per method (three draws of the images for $P = 25$, 100 and 500).
  \emph{CL, removed} is the accuracy of cordial agents when the message term is removed. \emph{CL $<$ ISO} is the fraction of pairs of runs on the same draw in which cordial learning
  has the lower test loss. A superscript gives the number of runs of the centralized network that did
  not learn (test loss was $\ln 10$).}
  \label{tab:common}
\end{table}

\FloatBarrier
\section{Conclusions}
We studied a distributed learning scenario in which agents' data are correlated. This generalizes the federated learning scenario where agents have independent data. To allow agents to learn from each other, we propose cordial learning, where agents send their embedded outputs to others. The agents train layers to learn how to extract valuable information from the embedded outputs of other agents. For a linear model where the label of an agent is a linear mixture of the inputs of all agents, under mild assumptions, we proved that cordial learning converges to models that minimize the sum of losses of all agents. Experiments on multi-digit MNIST with deep neural networks show that cordial learning performs well even beyond the theoretical model we analyzed. 

Our analysis focused on asymptotic convergence to the optimal models. The question of analyzing and optimizing the convergence rate remains open. The related question of proving optimal regret bounds in online learning settings is also open. 

Cordial learning is a first step towards studying distributed learning scenarios with correlated data. A challenging open question from this perspective is designing a cordial learning scheme where agents learn how to efficiently compress the information they send based on the loss functions of the receiving agents. 

\section*{Acknowledgments}
This research was supported by the Koret Foundation grant for Smart Cities and Digital Living 2030.

\appendix
\renewcommand{\theHsection}{app.\arabic{section}}
\section{Proof of Lemma \ref{lem:expectation}} \label{sec:Proof of Lemma 3}

We start by writing the stochastic loss:
\begin{multline}
    L_{i}=\left(Y_{i}-M_{i}\tilde{Z}_{i,0}^{i}-\sum_{j=1,j\neq i}^{N}\Lambda_{j}^{i}M_{j}\tilde{Z}_{j,0}^{i}\right)^{T}\left(Y_{i}-M_{i}\tilde{Z}_{i,0}^{i}-\sum_{j=1,j\neq i}^{N}\Lambda_{j}^{i}M_{j}\tilde{Z}_{j,0}^{i}\right)=\\
    Y_{i}^{T}Y_{i}-Y_{i}^{T}M_{i}\tilde{Z}_{i,0}^{i}-Y_{i}^{T}\sum_{j=1,j\neq i}^{N}\Lambda_{j}^{i}M_{j}\tilde{Z}_{j,0}^{i}-(M_{i}\tilde{Z}_{i,0}^{i})^{T}Y_{i}+(M_{i}\tilde{Z}_{i,0}^{i})^{T}M_{i}\tilde{Z}_{i,0}^{i} \\+ (M_{i}\tilde{Z}_{i,0}^{i})^{T}\sum_{j=1,j\neq i}^{N}\Lambda_{j}^{i}M_{j}\tilde{Z}_{j,0}^{i}
    -\left(\sum_{j=1,j\neq i}^{N}\Lambda_{j}^{i}M_{j}\tilde{Z}_{j,0}^{i}\right)^{T}Y_{i}+\left(\sum_{j=1,j\neq i}^{N}\Lambda_{j}^{i}M_{j}\tilde{Z}_{j,0}^{i}\right)^{T}M_{i}\tilde{Z}_{i,0}^{i} \\ +\left(\sum_{j=1,j\neq i}^{N}\Lambda_{j}^{i}M_{j}\tilde{Z}_{j,0}^{i}\right)^{T}\sum_{j=1,j\neq i}^{N}\Lambda_{j}^{i}M_{j}\tilde{Z}_{j,0}^{i}.\label{eq:27a}
\end{multline}

Thus the gradient of $L_i$ with respect to $M_i$ is
\begin{align}
\frac{\partial L_i}{\partial M_i}
&= \frac{\partial}{\partial M_i}\Big\{-2Y_i^T M_i\tilde{Z}^i_{i,0} + (M_i\tilde{Z}^i_{i,0})^T M_i\tilde{Z}^i_{i,0} + 2\Big(\sum_{j=1,j\neq i}^{N}\Lambda^i_j M_j\tilde{Z}^i_{j,0}\Big)^T M_i\tilde{Z}^i_{i,0}\Big\} \nonumber\\
&= -2Y_i\tilde{Z}^{iT}_{i,0} + 2M_i\tilde{Z}^i_{i,0}\tilde{Z}^{iT}_{i,0} + 2\Big(\sum_{j=1,j\neq i}^{N}\Lambda^i_j M_j\tilde{Z}^i_{j,0}\Big)\tilde{Z}^{iT}_{i,0}, \label{eq:28}
\end{align}
where we used the derivative rules $\frac{\partial(Y^TMX)}{\partial M} = YX^T$ and $\frac{\partial(X^TA^TAX)}{\partial A} = 2AXX^T$. The stochastic gradient $\boldsymbol{g}_{M_i}$ in \eqref{eq:15} is \eqref{eq:28} with $\Lambda^i_j$ replaced by its projection $\Pi_t\left[\Lambda^i_j(t)\right]$.

The gradient of $L_i$ with respect to $\Lambda^i_j$ is
\begin{equation*}
\frac{\partial L_i}{\partial \Lambda^i_j} = -2\Big(Y_i - M_i\tilde{Z}^i_{i,0} - \sum_{l=1,l\neq i}^{N}\Lambda^i_l M_l\tilde{Z}^i_{l,0}\Big)\big(M_j\tilde{Z}^i_{j,0}\big)^T.
\end{equation*}
Algorithm~1 replaces $\tilde{Z}^i_{j,0}$ in the last factor by $\tilde{Z}^i_{j,1}$, which is computed from the independent batch $B^1_t$. This gives
\begin{equation}
\boldsymbol{g}_{\Lambda^i_j} = -2Y_i(M_j\tilde{Z}^i_{j,1})^T + 2M_i\tilde{Z}^i_{i,0}(M_j\tilde{Z}^i_{j,1})^T + 2\Big(\sum_{l=1,l\neq i}^{N}\Lambda^i_l M_l\tilde{Z}^i_{l,0}\Big)(M_j\tilde{Z}^i_{j,1})^T. \label{eq:29}
\end{equation}
With this replacement, the two covariance estimates in each product are independent, which makes \eqref{eq:29} unbiased (see \eqref{eq:32a}).

The expected gradients are as follows:
\begin{multline}
    \mathbb{E}\left\{ \boldsymbol{g}_{M_i} \right\} =\mathbb{E}\left\{ 2\left(M_{i}-A_{i}\right)\tilde{Z}_{i,0}^{i}\tilde{Z}_{i,0}^{iT}+2\sum_{j=1,j\neq i}^{N}\left(\Lambda_{j}^{i}M_{j}\tilde{Z}_{j,0}^{i}-B_{j}^{i}X_{j}\right)\tilde{Z}_{i,0}^{iT}\right\} = \\
    2\left(M_{i}-A_{i}-\sum_{j=1,j\neq i}^{N}B_{j}^{i}\Cov\left(X_{j},X_{i}\right)\right) \label{eq:30a}
\end{multline}
where we used 
\begin{equation}
\mathbb{E}
\left\{
\tilde{Z}_{j,0}^iX_i^T
\right\}
=
\mathbb{E}
\left\{
\left(
X_j-\tilde{R}_{j,0}^iX_i
\right)X_i^T
\right\}
\nonumber=
\mathbb{E}\{X_jX_i^T\}
-
\mathbb{E}\{\tilde{R}_{j,0}^i\}
\mathbb{E}\{X_iX_i^T\}
\nonumber=0,
\label{eq:31}
\end{equation}
and $\mathbb{E}\left\{ X_{j}\tilde{Z}_{i,0}^{iT}\right\} =\Cov\left(X_{j},X_{i}\right)$. 

Next,
\begin{align}
\mathbb{E}
\left\{
\boldsymbol{g}_{\Lambda_j^i}
\right\}
&=
\mathbb{E}
\Bigg\{
\Bigg(
2(M_i-A_i)X_i
+
2\sum_{k=1,k\neq i}^{N}
\left(
\Lambda_k^iM_k\tilde{Z}_{k,0}^i
-
B_k^iX_k
\right)
\Bigg)
\left(\tilde{Z}_{j,1}^i\right)^T
M_j^T
\Bigg\}
\nonumber\\
&=
2\sum_{k=1,k\neq i}^{N}
\left(
\Lambda_k^iM_k-B_k^i
\right)
C_{k,j}M_j^T, 
\label{eq:31a}
\end{align}
where we used 
\begin{multline}
\mathbb{E}
\left\{
\tilde{Z}_{k,0}^i
\left(\tilde{Z}_{j,1}^i\right)^T
\right\}
=
\mathbb{E}
\left\{
\left(
X_k-\tilde{R}_{k,0}^iX_i
\right)
\left(
X_j-\tilde{R}_{j,1}^iX_i
\right)^T
\right\}=\\
\mathbb{E}\{X_kX_j^T\}
-
\mathbb{E}\{X_kX_i^T\}
\mathbb{E}\{(\tilde{R}_{j,1}^i)^T\}
-
\mathbb{E}\{\tilde{R}_{k,0}^i\}
\mathbb{E}\{X_iX_j^T\}
+
\mathbb{E}\{\tilde{R}_{k,0}^i\}
\mathbb{E}\{X_iX_i^T\}
\mathbb{E}\{(\tilde{R}_{j,1}^i)^T\}=\\
\mathbb{E}
\left\{
X_k
\left(
X_j-\tilde{R}_{j,1}^iX_i
\right)^T
\right\}
=\mathbb{E}\{X_kX_j^T\}
-
\operatorname{Cov}(X_k,X_i)
\operatorname{Cov}(X_j,X_i)^T
=C_{k,j}.
\label{eq:32a}
\end{multline}

Here we used that $\tilde{R}_{k,0}^i$ is computed using the
leave-one-out samples from $B_t^0$, while
$\tilde{R}_{j,1}^i$ is computed using the independent batch
$B_t^1$. Hence, $\tilde{R}_{k,0}^i$,
$\tilde{R}_{j,1}^i$, and the given sample are mutually independent. Moreover,
\[
\mathbb{E}\{\tilde{R}_{k,0}^i\}
=
\operatorname{Cov}(X_k,X_i),
\qquad
\mathbb{E}\{\tilde{R}_{j,1}^i\}
=
\operatorname{Cov}(X_j,X_i).
\]

\section{Nonconvex Objective} \label{sec:Nonconvex Objective}

We show that the global objective in~(7) is nonconvex,
even in the two-agent scalar case. Let $d=m=1$, and let $X_1$ and
$X_2$ be zero-mean random variables satisfying $\mathbb{E}[X_1^2]=\mathbb{E}[X_2^2]=1$ and $\mathbb{E}[X_1X_2]=0$.
    
Consider the noiseless labels $Y_1=X_2$ and  $Y_2=X_2$. Then $A_1^*=0$ and $A_2^*=1$, so Assumption~1 is satisfied by taking
$\Lambda_2^{1*}=1$ and $\Lambda_1^{2*}=0$.

For two scalar agents, the global objective is
\[
\begin{aligned}
F(M_1,M_2,\Lambda_1,\Lambda_2)
={}&
\mathbb{E}\!\left[
    \left(
        Y_1-M_1X_1-\Lambda_1M_2X_2
    \right)^2
\right] \\
&+
\mathbb{E}\!\left[
    \left(
        Y_2-M_2X_2-\Lambda_2M_1X_1
    \right)^2
\right]
\end{aligned}
\]
where $\Lambda_1 = \Lambda^1_2$ and $\Lambda_2 = \Lambda^2_1$. 

Next we restrict this objective to the affine subspace $M_1 = 0$ and $\Lambda_2 = 0$. On this subspace,
\[
\begin{aligned}
h(M_2,\Lambda_1)
&\triangleq F(0,M_2,\Lambda_1,0) \\
&=
\mathbb{E}\!\left[
    \left(X_2-\Lambda_1 M_2X_2\right)^2
\right]
+
\mathbb{E}\!\left[
    \left(X_2-M_2X_2\right)^2
\right] \\
&=
(1-\Lambda_1 M_2)^2+(1-M_2)^2.
\end{aligned}
\]
The Hessian of $h$ is
\[
    \nabla^2 h(M_2,\Lambda_1)
    =
    \begin{pmatrix}
        2+2\Lambda_1^2 & -2+4\Lambda_1 M_2 \\
        -2+4\Lambda_1 M_2 & 2M_2^2
    \end{pmatrix}.
\]
At $(M_2,\Lambda_1)=(0,0)$,
\[
    \nabla^2 h(0,0)
    =
    \begin{pmatrix}
        2 & -2 \\
        -2 & 0
    \end{pmatrix},
\]
whose determinant is $-4$. Hence,
$\nabla^2 h(0,0)$ is indefinite, and $h$ is not convex.

Since the restriction of a convex function to any affine subspace
must be convex, the full objective
$F(M_1,M_2,\Lambda_1,\Lambda_2)$ is not convex.
Therefore, the global objective in~(7) is nonconvex. 

\section{Proof of Theorem \ref{thm:main}} \label{sec:Proof of Theorem}
We start by proving the convergence of $M_i(t)$. We have 
 
\begin{multline}
\mathbb{E}\left\{ \sum_{i=1}^{N}\left\Vert A_{i}^{*}-M_{i}\left(t+1\right)\right\Vert ^{2}\mid\mathcal{F}_{t}\right\} =\mathbb{E}\left\{ \sum_{i=1}^{N}\left\Vert A_{i}^{*}-M_{i}\left(t\right)+\varepsilon\left(t\right)\boldsymbol{g}_{M_{i}}\left(t\right)\right\Vert ^{2}\mid\mathcal{F}_{t}\right\} = \\ 
\sum_{i=1}^{N}\left\Vert A_{i}^{*}-M_{i}\left(t\right)\right\Vert ^{2}+2\varepsilon\left(t\right)\sum_{i=1}^{N}\mathbb{E}\left\{ \left\langle \boldsymbol{g}_{M_{i}}\left(t\right),A_{i}^{*}-M_{i}\left(t\right)\right\rangle \mid\mathcal{F}_{t}\right\} +\varepsilon^{2}\left(t\right)\sum_{i=1}^{N}\mathbb{E}\left\{ \left\Vert \boldsymbol{g}_{M_{i}}\left(t\right)\right\Vert ^{2}\mid\mathcal{F}_{t}\right\} \underset{\left(a\right)}{\leq} \\ 
\left(1+K_{1}\varepsilon^{2}\left(t\right)\rho^{2}\left(t\right)\right)\sum_{i=1}^{N}\left\Vert A_{i}^{*}-M_{i}\left(t\right)\right\Vert ^{2}-4\varepsilon\left(t\right)\sum_{i=1}^{N}\left\Vert A_{i}^{*}-M_{i}\left(t\right)\right\Vert ^{2}+K_{2}\varepsilon^{2}\left(t\right)\rho^{2}\left(t\right)\label{eq:32}
\end{multline}
where (a) uses Lemma \ref{lem:expectation} and Lemma \ref{lem:variance}. We also used that $M_{i}(t)$ is measurable by $\mathcal{F}_{t}$.

Given the inequality above, we can apply the Robbins-Siegmund Theorem
\citep{robbins1971convergence} with (using their notation)
$z_{t}=\sum_{i=1}^{N}\left\Vert A_{i}^{*}-M_{i}\left(t\right)\right\Vert ^{2}$, $\beta_{t}=K_{1}\varepsilon^{2}\left(t\right)\rho^{2}(t)$, 
$\xi_{t}=K_{2}\varepsilon^{2}\left(t\right)\rho^{2}(t)$ and 
$\zeta_{t}=4\varepsilon\left(t\right)\sum_{i=1}^{N}\left\Vert A_{i}^{*}-M_{i}\left(t\right)\right\Vert ^{2}$.
Hence, with probability 1, for all $i$:
\begin{equation}
\sum_{t=1}^{\infty}\varepsilon\left(t\right)\sum_{i=1}^{N}\left\Vert A_{i}^{*}-M_{i}\left(t\right)\right\Vert ^{2}<\infty\label{eq:33}
\end{equation}
and $\sum_{i=1}^{N}\left\Vert A_{i}^{*}-M_{i}\left(t\right)\right\Vert ^{2}$ converges to some random variable, which must be zero with probability 1, since any positive probability to be positive would contradict \eqref{eq:33}, since $\sum_{t=1}^{\infty}\varepsilon\left(t\right)=\infty$ by assumption.

Now let $f_{i}\left(\Lambda_{1}^{i},\ldots,\Lambda_{N}^{i}\right)\triangleq\mathbb{E}\left\{ \left\Vert \sum_{k=1,k\neq i}^{N}\left(\Lambda_{k}^{i}A_{k}^{*}-B_{k}^{i}\right)Z_{k}^{i}\right\Vert ^{2}\right\}$. We have for some constant $K_0>0$
\begin{multline}
\mathbb{E}\left\{ \sum_{j\neq i}\left\Vert \Lambda_{j}^{i}\left(t+1\right)-\Lambda_{j}^{i*}\right\Vert ^{2}\mid\mathcal{F}_{t}\right\} =\mathbb{E}\left\{ \sum_{j\neq i}\left\Vert \Lambda_{j}^{i}\left(t\right)-\eta\left(t\right)\boldsymbol{g}_{\Lambda_{j}^{i}}\left(t\right)-\Lambda_{j}^{i*}\right\Vert ^{2}\mid\mathcal{F}_{t}\right\} = \\
\sum_{j\neq i}\left\Vert \Lambda_{j}^{i}\left(t\right)-\Lambda_{j}^{i*}\right\Vert ^{2}-2\eta\left(t\right)\mathbb{E}\left\{ \sum_{j\neq i}\left\langle \boldsymbol{g}_{\Lambda_{j}^{i}}\left(t\right),\Lambda_{j}^{i}\left(t\right)-\Lambda_{j}^{i*}\right\rangle \mid\mathcal{F}_{t}\right\} +\eta^{2}\left(t\right)\mathbb{E}\left\{ \sum_{j\neq i}\left\Vert \boldsymbol{g}_{\Lambda_{j}^{i}}\left(t\right)\right\Vert ^{2}\mid\mathcal{F}_{t}\right\} \underset{\left(a\right)}{\leq}\\
\sum_{j\neq i}\left\Vert \Lambda_{j}^{i}\left(t\right)-\Lambda_{j}^{i*}\right\Vert ^{2}-4\eta\left(t\right)\sum_{j\neq i}\left\langle \sum_{k\neq i}\left(\Lambda_{k}^{i}\left(t\right)M_{k}\left(t\right)-B_{k}^{i}\right)C_{k,j}M_{j}^{T}\left(t\right),\Lambda_{j}^{i}\left(t\right)-\Lambda_{j}^{i*}\right\rangle \\
+K_0\eta^{2}\left(t\right)\sum_{l=1}^{N}\left(1+\left\Vert \Lambda_{l}^{i}\left(t\right)\right\Vert ^{2}\right)\sum_{k=1}^{N}\left(1+2\left\Vert M_{k}\left(t\right)-A_{k}^{*}\right\Vert ^{2}+\left\Vert M_{k}\left(t\right)-A_{k}^{*}\right\Vert ^{4}\right)
\label{eq:34}
\end{multline}
where (a) follows from Lemma \ref{lem:expectation} and Lemma \ref{lem:variance}.

To deal with the bias term in \eqref{eq:34}, we write:
\begin{multline}
    \sum_{j\neq i}\left\langle \sum_{k\neq i}\left(\Lambda_{k}^{i}\left(t\right)M_{k}\left(t\right)-B_{k}^{i}\right)C_{k,j}M_{j}^{T}\left(t\right),\Lambda_{j}^{i}\left(t\right)-\Lambda_{j}^{i*}\right\rangle = \\
\sum_{j\neq i}\left\langle \sum_{k\neq i}\left(\Lambda_{k}^{i}\left(t\right)A_{k}^{*}-B_{k}^{i}\right)C_{k,j}A_{j}^{*T},\Lambda_{j}^{i}\left(t\right)-\Lambda_{j}^{i*}\right\rangle + \\
\sum_{j\neq i}\left\langle \sum_{k\neq i}B_{k}^{i}C_{k,j}\left(A_{j}^{*T}-M_{j}^{T}\left(t\right)\right),\Lambda_{j}^{i}\left(t\right)-\Lambda_{j}^{i*}\right\rangle + \\
\sum_{j\neq i}\left\langle \sum_{k\neq i}\Lambda_{k}^{i}\left(t\right)\left(M_{k}\left(t\right)C_{k,j}M_{j}^{T}\left(t\right)-A_{k}^{*}C_{k,j}A_{j}^{*T}\right),\Lambda_{j}^{i}\left(t\right)-\Lambda_{j}^{i*}\right\rangle. \label{eq:35}
\end{multline}

 For the first term in \eqref{eq:35}, we use the convexity of $f_i$ to obtain

\begin{multline}
\sum_{j\neq i}\left\langle \sum_{k=1,k\neq i}^{N}\left(\Lambda_{k}^{i}\left(t\right)A_{k}^{*}-B_{k}^{i}\right)C_{k,j}A_{j}^{*T},\Lambda_{j}^{i}\left(t\right)-\Lambda_{j}^{i*}\right\rangle =\\
\sum_{j\neq i}\left\langle \frac{1}{2}\nabla_{\Lambda_{j}^{i}}f_{i}\left(\Lambda_{1}^{i}\left(t\right),\ldots,\Lambda_{N}^{i}\left(t\right)\right),\Lambda_{j}^{i}\left(t\right)-\Lambda_{j}^{i*}\right\rangle \geq \\ \frac{1}{2}f_{i}\left(\Lambda_{1}^{i}\left(t\right),\ldots,\Lambda_{N}^{i}\left(t\right)\right)-\frac{1}{2}f_{i}\left(\Lambda_{1}^{i*},\ldots,\Lambda_{N}^{i*}\right)=\frac{1}{2}f_{i}\left(\Lambda_{1}^{i}\left(t\right),\ldots,\Lambda_{N}^{i}\left(t\right)\right) \label{eq:36}
\end{multline}
since $f_{i}\left(\Lambda_{1}^{i*},\ldots,\Lambda_{N}^{i*}\right)=0$ by definition of $\Lambda_{j}^{i*}$.

For the second term in \eqref{eq:35} we  use:
\begin{multline}
    -\sum_{j\neq i}\left\langle \sum_{k\neq i}B_{k}^{i}C_{k,j}\left(A_{j}^{*T}-M_{j}^{T}\left(t\right)\right),\Lambda_{j}^{i}\left(t\right)-\Lambda_{j}^{i*}\right\rangle \underset{\left(a\right)}{\leq}\\
\sum_{j\neq i}\sum_{k\neq i}\left\Vert B_{k}^{i}C_{k,j}\left(A_{j}^{*T}-M_{j}^{T}\left(t\right)\right)\right\Vert \left\Vert \Lambda_{j}^{i}\left(t\right)-\Lambda_{j}^{i*}\right\Vert \leq\\
C_{5}\sum_{j\neq i}\left\Vert M_{j}\left(t\right)-A_{j}^{*}\right\Vert \left\Vert \Lambda_{j}^{i}\left(t\right)-\Lambda_{j}^{i*}\right\Vert \underset{\left(b\right)}{\leq} \\ C_{5}\sum_{j\neq i}\left\Vert M_{j}\left(t\right)-A_{j}^{*}\right\Vert \left\Vert \Lambda_{j}^{i}\left(t\right)-\Lambda_{j}^{i*}\right\Vert ^{2}+C_{5}\sum_{j\neq i}\left\Vert M_{j}\left(t\right)-A_{j}^{*}\right\Vert  \label{eq:39}
\end{multline}
where $\left(a\right)$ follows from the Cauchy-Schwarz inequality and the triangle inequality, and $\left(b\right)$ uses $x \leq x^2 + 1$.

For the third term in \eqref{eq:35} we use:
\begin{multline}
-\sum_{j\neq i}\left\langle \sum_{k\neq i}\Lambda_{k}^{i}\left(t\right)\left(M_{k}\left(t\right)C_{k,j}M_{j}^{T}\left(t\right)-A_{k}^{*}C_{k,j}A_{j}^{*T}\right),\Lambda_{j}^{i}\left(t\right)-\Lambda_{j}^{i*}\right\rangle \underset{\left(a\right)}{\leq}\\
\sum_{j\neq i}\left\Vert \Lambda_{j}^{i}\left(t\right)-\Lambda_{j}^{i*}\right\Vert \sum_{k\neq i}\left\Vert \Lambda_{k}^{i}\left(t\right)\left(M_{k}\left(t\right)-A_{k}^{*}\right)C_{k,j}A_{j}^{*T}\right\Vert \\
+\sum_{j\neq i}\left\Vert \Lambda_{j}^{i}\left(t\right)-\Lambda_{j}^{i*}\right\Vert \sum_{k\neq i}\left\Vert \Lambda_{k}^{i}\left(t\right)M_{k}\left(t\right)C_{k,j}\left(M_{j}^{T}\left(t\right)-A_{j}^{*T}\right)\right\Vert \leq\\
\sum_{j\neq i}\sum_{k\neq i}\left\Vert \left(\Lambda_{k}^{i}\left(t\right)-\Lambda_{k}^{i*}\right)\left(M_{k}\left(t\right)-A_{k}^{*}\right)C_{k,j}A_{j}^{*T}\right\Vert \left\Vert \Lambda_{j}^{i}\left(t\right)-\Lambda_{j}^{i*}\right\Vert \\
+\sum_{j\neq i}\sum_{k\neq i}\left\Vert \Lambda_{k}^{i*}\left(M_{k}\left(t\right)-A_{k}^{*}\right)C_{k,j}A_{j}^{*T}\right\Vert \left\Vert \Lambda_{j}^{i}\left(t\right)-\Lambda_{j}^{i*}\right\Vert \\
+\sum_{j\neq i}\sum_{k\neq i}\left\Vert \left(\Lambda_{k}^{i}\left(t\right)-\Lambda_{k}^{i*}\right)M_{k}\left(t\right)C_{k,j}\left(M_{j}^{T}\left(t\right)-A_{j}^{*T}\right)\right\Vert \left\Vert \Lambda_{j}^{i}\left(t\right)-\Lambda_{j}^{i*}\right\Vert \\
+\sum_{j\neq i}\sum_{k\neq i}\left\Vert \Lambda_{k}^{i*}M_{k}\left(t\right)C_{k,j}\left(M_{j}^{T}\left(t\right)-A_{j}^{*T}\right)\right\Vert \left\Vert \Lambda_{j}^{i}\left(t\right)-\Lambda_{j}^{i*}\right\Vert \underset{\left(b\right)}{\leq}\\
C_{1}\sum_{k=1}^{N}\left\Vert M_{k}\left(t\right)-A_{k}^{*}\right\Vert \sum_{l\neq i}\left\Vert \Lambda_{l}^{i}\left(t\right)-\Lambda_{l}^{i*}\right\Vert +C_{2}\sum_{k=1}^{N}\left\Vert M_{k}\left(t\right)-A_{k}^{*}\right\Vert \sum_{l\neq i}\left\Vert \Lambda_{l}^{i}\left(t\right)-\Lambda_{l}^{i*}\right\Vert ^{2}+ \\
C_{3}\sum_{k=1}^{N}\left\Vert M_{k}\left(t\right)-A_{k}^{*}\right\Vert ^{2}\sum_{l\neq i}\left\Vert \Lambda_{l}^{i}\left(t\right)-\Lambda_{l}^{i*}\right\Vert ^{2}+C_{4}\sum_{k=1}^{N}\left\Vert M_{k}\left(t\right)-A_{k}^{*}\right\Vert ^{2}\sum_{l\neq i}\left\Vert \Lambda_{l}^{i}\left(t\right)-\Lambda_{l}^{i*}\right\Vert \leq \\
 C_{5}\sum_{k=1}^{N}\left\Vert M_{k}\left(t\right)-A_{k}^{*}\right\Vert \sum_{l\neq i}\left\Vert \Lambda_{l}^{i}\left(t\right)-\Lambda_{l}^{i*}\right\Vert ^{2}+ C_{1}\sum_{k=1}^{N}\left\Vert M_{k}\left(t\right)-A_{k}^{*}\right\Vert + \\
C_{6}\sum_{k=1}^{N}\left\Vert M_{k}\left(t\right)-A_{k}^{*}\right\Vert ^{2}\sum_{l\neq i}\left\Vert \Lambda_{l}^{i}\left(t\right)-\Lambda_{l}^{i*}\right\Vert ^{2}+C_{4}\sum_{k=1}^{N}\left\Vert M_{k}\left(t\right)-A_{k}^{*}\right\Vert ^{2} \label{eq:37}
\end{multline}
where $\left(a\right)$ follows from the Cauchy-Schwarz inequality and the triangle inequality. In (b) we used that, for $a_l = \|\Lambda^i_l(t) - \Lambda^{i*}_l\|$ ($l \neq i$), $a_i = 0$, $b_l = \|M_l(t) - A^*_l\|$, $S_a = \sum_l a_l$, $S_b = \sum_l b_l$, $S_{a^2} = \sum_l a_l^2$ and $S_{b^2} = \sum_l b_l^2$, Cauchy-Schwarz and the L1-L2 norm inequality yield
\begin{multline}
\sum_{k,l} a_k a_l b_k + \sum_{k,l} a_k b_l b_k + \sum_{k,l} a_k b_l + \sum_{k,l} a_k a_l b_k b_l
= \\ (S_a + S_b) \sum_k a_k b_k + S_a S_b + \Big(\sum_k a_k b_k\Big)^2 \le \\
(S_a + S_b)\sqrt{S_{a^2} S_{b^2}} + S_a S_b + S_{a^2} S_{b^2} \le
\sqrt{N}\left(\sqrt{S_{a^2}} + \sqrt{S_{b^2}}\right)\sqrt{S_{a^2} S_{b^2}} + S_a S_b + S_{a^2} S_{b^2} \le \\
\sqrt{N} S_{a^2} S_b + \sqrt{N} S_a S_{b^2} + S_a S_b + S_{a^2} S_{b^2}. \label{eq:38}
\end{multline}

Plugging \eqref{eq:35}, \eqref{eq:36}, \eqref{eq:37} and \eqref{eq:39} into \eqref{eq:34}, we obtain 

\begin{multline}
\mathbb{E}\left\{ \sum_{j\neq i}\left\Vert \Lambda_{j}^{i}\left(t+1\right)-\Lambda_{j}^{i*}\right\Vert ^{2}\mid\mathcal{F}_{t}\right\} \leq \\
\sum_{j\neq i}\left\Vert \Lambda_{j}^{i}\left(t\right)-\Lambda_{j}^{i*}\right\Vert ^{2} -2\eta\left(t\right)f_{i}\left(\Lambda_{1}^{i}\left(t\right),\ldots,\Lambda_{N}^{i}\left(t\right)\right) \\
+C_{1}\sum_{k=1}^{N}\eta\left(t\right)\left(\left\Vert M_{k}\left(t\right)-A_{k}^{*}\right\Vert + \left\Vert M_{k}\left(t\right)-A_{k}^{*}\right\Vert^{2} +\eta\left(t\right)\left(\left\Vert M_{k}\left(t\right)-A_{k}^{*}\right\Vert ^{4}+1\right)\right)\sum_{j\neq i}\left\Vert \Lambda_{j}^{i}\left(t\right)-\Lambda_{j}^{i*}\right\Vert ^{2}\\+ C_{2}\sum_{k=1}^{N}\eta\left(t\right)\left(\left\Vert M_{k}\left(t\right)-A_{k}^{*}\right\Vert + \left\Vert M_{k}\left(t\right)-A_{k}^{*}\right\Vert^{2} +\eta\left(t\right)\left(\left\Vert M_{k}\left(t\right)-A_{k}^{*}\right\Vert ^{4}+1\right)\right).
\label{eq:40}
\end{multline}

Given \eqref{eq:40}, we can again apply the Robbins-Siegmund Theorem with 
$z_{t}=\sum_{j\neq i}\left\Vert \Lambda_{j}^{i}\left(t\right)-\Lambda_{j}^{i*}\right\Vert ^{2}$, 

\begin{equation}
    \beta_{t}= C_{1}\sum_{k=1}^{N}\eta\left(t\right)\left(\left\Vert M_{k}\left(t\right)-A_{k}^{*}\right\Vert + \left\Vert M_{k}\left(t\right)-A_{k}^{*}\right\Vert^{2} +\eta\left(t\right)+\eta\left(t\right)\left\Vert M_{k}\left(t\right)-A_{k}^{*}\right\Vert ^{4}\right) \label{eq:41}
\end{equation}
\begin{equation}
    \xi_{t}= C_{2}\sum_{k=1}^{N}\eta\left(t\right)\left(\left\Vert M_{k}\left(t\right)-A_{k}^{*}\right\Vert + \left\Vert M_{k}\left(t\right)-A_{k}^{*}\right\Vert^{2} +\eta\left(t\right)+\eta\left(t\right)\left\Vert M_{k}\left(t\right)-A_{k}^{*}\right\Vert ^{4}\right) \label{eq:42} 
\end{equation}
 and $\zeta_{t}=2\eta\left(t\right)f_{i}\left(\Lambda_{1}^{i}\left(t\right),\ldots,\Lambda_{N}^{i}\left(t\right)\right).$

We know from \eqref{eq:33}  that $\sum_{t=1}^{\infty}\varepsilon\left(t\right)\left\Vert A_{i}^{*}-M_{i}\left(t\right)\right\Vert ^{2}<\infty$. Thus, using the assumption that $\sum_{t=1}^{\infty}\frac{\eta^{2}\left(t\right)}{\varepsilon\left(t\right)}<\infty$:

\begin{multline}
    \sum_{t=1}^{\infty}\eta\left(t\right)\left\Vert A_{i}^{*}-M_{i}\left(t\right)\right\Vert =\sum_{t=1}^{\infty}\frac{\eta\left(t\right)}{\sqrt{\varepsilon\left(t\right)}}\sqrt{\varepsilon\left(t\right)}\left\Vert A_{i}^{*}-M_{i}\left(t\right)\right\Vert \leq \\ \sum_{t=1}^{\infty}\frac{\eta^{2}\left(t\right)}{\varepsilon\left(t\right)}+\sum_{t=1}^{\infty}\varepsilon\left(t\right)\left\Vert A_{i}^{*}-M_{i}\left(t\right)\right\Vert ^{2}<\infty.\label{eq:43}
\end{multline}

Then, $\sum_{t=1}^{\infty}\eta\left(t\right)\left\Vert A_{i}^{*}-M_{i}\left(t\right)\right\Vert ^{2}$ and $\sum_{t=1}^{\infty}\eta^{2}\left(t\right)\left\Vert A_{i}^{*}-M_{i}\left(t\right)\right\Vert ^{4}$ are also finite since $\left\Vert A_{i}^{*}-M_{i}\left(t\right)\right\Vert$ converges to zero with probability 1, and $\sum_{t=1}^{\infty} \eta^2\left(t\right)\leq \left(\sup_t \varepsilon\left(t\right)\right) \sum_{t=1}^{\infty}\frac{\eta^{2}\left(t\right)}{\varepsilon\left(t\right)}<\infty$.

Hence, the Robbins-Siegmund Theorem implies that with probability 1, for all $i$:
\begin{equation}
\sum_{t=1}^{\infty}\eta\left(t\right)f_{i}\left(\Lambda_{1}^{i}\left(t\right),\ldots,\Lambda_{N}^{i}\left(t\right)\right)<\infty\label{eq:50}
\end{equation}
and $\sum_{j\neq i}\left\Vert \Lambda_{j}^{i}\left(t\right)-\Lambda_{j}^{i*}\right\Vert ^{2}$
converges to some random variable $V$ as $t\rightarrow\infty$. In particular, $\{\boldsymbol{\Lambda}^i(t)\}_t$ is bounded.

Define the closed set
\[
    \mathcal{S}_i
    \triangleq
    \left\{
        \boldsymbol{\Lambda}^i:
        f_i(\boldsymbol{\Lambda}^i)=0
    \right\}
\]
which is nonempty by Assumption~1. Let
$\mathcal{S}_i^0$ be a countable dense subset of $\mathcal{S}_i$.
Applying the preceding Robbins-Siegmund argument to every fixed
$\boldsymbol{\Lambda}^{i*}\in\mathcal{S}_i^0$, and taking the
intersection of the resulting probability-one events, shows that,
with probability one, $\left\|\boldsymbol{\Lambda}^i(t)-\boldsymbol{\Lambda}^{i*}\right\|$ converges for every
$\boldsymbol{\Lambda}^{i*}\in\mathcal{S}_i^0$. The same conclusion
holds for every $\boldsymbol{\Lambda}^{i*}\in\mathcal{S}_i$ since for any 
$\boldsymbol{\Lambda}^{i*}\in\mathcal{S}_i$ and any sequence
$\boldsymbol{\Lambda}^{i*,r}\in\mathcal{S}_i^0$ converging to
$\boldsymbol{\Lambda}^{i*}$,
\[
\left|
 \left\|\boldsymbol{\Lambda}^i(t)-\boldsymbol{\Lambda}^{i*}\right\|
 -
 \left\|\boldsymbol{\Lambda}^i(t)-\boldsymbol{\Lambda}^{i*,r}\right\|
\right|
\leq
\left\|\boldsymbol{\Lambda}^{i*}
      -\boldsymbol{\Lambda}^{i*,r}\right\|,
\]
uniformly in $t$.

Moreover, since $\sum_t\eta(t)=\infty$,
\eqref{eq:50} implies that
\[
    \liminf_{t\to\infty}
    f_i\!\left(\boldsymbol{\Lambda}^i(t)\right)=0.
\]
Consequently, there exists a subsequence $\{t_r\}_r$ such that $ f_i\!\left(\boldsymbol{\Lambda}^i(t_r)\right)\rightarrow 0$. By boundedness, this subsequence has a further subsequence, which we
denote again by $\{t_r\}_r$, such that $\boldsymbol{\Lambda}^i(t_r)
    \rightarrow \overline{\boldsymbol{\Lambda}}^i$. Since $f_i$ is continuous, $f_i\!\left(\overline{\boldsymbol{\Lambda}}^i\right)=0$ and hence $\overline{\boldsymbol{\Lambda}}^i\in\mathcal{S}_i$.

As shown above,
$\|\boldsymbol{\Lambda}^i(t)
-\overline{\boldsymbol{\Lambda}}^i\|$
has a limit. Along the subsequence $\{t_r\}_r$, it converges
to zero. Therefore, its full-sequence limit is zero, and $ \boldsymbol{\Lambda}^i(t)
    \rightarrow\overline{\boldsymbol{\Lambda}}^i$ with probability one. By continuity of $f_i$, it follows that $ f_i\!\left(\boldsymbol{\Lambda}^i(t)\right)\rightarrow f_i\!\left(\overline{\boldsymbol{\Lambda}}^i\right)=0$ with probability one.

Together with $M_k(t)\to A_k^*$ for every $k$, and the boundedness of
$\{\Lambda_k^i(t)\}_t$, this implies that the loss of each agent $i$, $L_i$, 
converges to its global minimum $\sigma^2$ with probability one.

\bibliography{DistributedTrainingCorrelatedData}

@Article{bistritz2020distillation,
  author  = {Bistritz, Ilai and Mann, Ariana and Bambos, Nicholas},
  journal = {Advances in Neural Information Processing Systems},
  title   = {Distributed distillation for on-device learning},
  year    = {2020},
  pages   = {22593--22604},
  volume  = {33},
}

@article{bistritz2020cooperative,
  title={Cooperative Multi-player Bandit Optimization},
  author={Bistritz, Ilai and Bambos, Nicholas},
  journal={Advances in Neural Information Processing Systems},
  volume={33},
  year={2020}
}

@Article{ratliff2016characterization,
  author    = {Ratliff, Lillian J and Burden, Samuel A and Sastry, S Shankar},
  journal   = {IEEE Transactions on Automatic Control},
  title     = {On the characterization of local {N}ash equilibria in continuous games},
  year      = {2016},
  number    = {8},
  pages     = {2301--2307},
  volume    = {61},
  publisher = {IEEE},
}

@article{mazumdar2020gradient,
  title={On Gradient-Based Learning in Continuous Games},
  author={Mazumdar, Eric and Ratliff, Lillian J and Sastry, S Shankar},
  journal={SIAM Journal on Mathematics of Data Science},
  volume={2},
  number={1},
  pages={103--131},
  year={2020},
  publisher={SIAM}
}

@inproceedings{chasnov2020convergence,
  title={Convergence analysis of gradient-based learning in continuous games},
  author={Chasnov, Benjamin and Ratliff, Lillian and Mazumdar, Eric and Burden, Samuel},
  booktitle={Uncertainty in Artificial Intelligence},
  pages={935--944},
  year={2020},
  organization={PMLR}
}

@Article{Hong2022,
  author    = {Hong, Lingzhou and Garcia, Alfredo and Eksin, Ceyhun},
  journal   = {Automatica},
  title     = {Distributed networked learning with correlated data},
  year      = {2022},
  pages     = {110134},
  volume    = {137},
  publisher = {Elsevier},
}

@article{Narang2023,
  title={Multiplayer performative prediction: Learning in decision-dependent games},
  author={Narang, Adhyyan and Faulkner, Evan and Drusvyatskiy, Dmitriy and Fazel, Maryam and Ratliff, Lillian J},
  journal={Journal of Machine Learning Research},
  volume={24},
  number={202},
  pages={1--56},
  year={2023}
}

@Article{Kairouz2021,
  author    = {Kairouz, Peter and McMahan, H Brendan and Avent, Brendan and Bellet, Aur{\'e}lien and Bennis, Mehdi and Bhagoji, Arjun Nitin and Bonawitz, Kallista and Charles, Zachary and Cormode, Graham and Cummings, Rachel and others},
  journal   = {Foundations and Trends{\textregistered} in Machine Learning},
  title     = {Advances and open problems in federated learning},
  year      = {2021},
  number    = {1--2},
  pages     = {1--210},
  volume    = {14},
  publisher = {Now Publishers, Inc.},
}

@article{goodfellow2020,
  title={Generative adversarial networks},
  author={Goodfellow, Ian and Pouget-Abadie, Jean and Mirza, Mehdi and Xu, Bing and Warde-Farley, David and Ozair, Sherjil and Courville, Aaron and Bengio, Yoshua},
  journal={Communications of the ACM},
  volume={63},
  number={11},
  pages={139--144},
  year={2020},
  publisher={ACM New York, NY, USA}
}

@incollection{robbins1971convergence,
  title={A convergence theorem for non negative almost supermartingales and some applications},
  author={Robbins, Herbert and Siegmund, David},
  booktitle={Optimizing Methods in Statistics},
  pages={233--257},
  year={1971},
  publisher={Academic Press}
}

@article{liu2024vertical,
  title={Vertical federated learning: Concepts, advances, and challenges},
  author={Liu, Yang and Kang, Yan and Zou, Tianyuan and Pu, Yanhong and He, Yuanqin and Ye, Xiaozhou and Ouyang, Ye and Zhang, Ya-Qin and Yang, Qiang},
  journal={IEEE Transactions on Knowledge and Data Engineering},
  volume={36},
  number={7},
  pages={3615--3634},
  year={2024},
  publisher={IEEE}
}

@inproceedings{perdomo2020performative,
  title={Performative prediction},
  author={Perdomo, Juan and Zrnic, Tijana and Mendler-D{\"u}nner, Celestine and Hardt, Moritz},
  booktitle={International Conference on Machine Learning},
  pages={7599--7609},
  year={2020},
  organization={PMLR}
}

@InProceedings{linplug,
  title = 	 {Plug-in Performative Optimization},
  author =       {Lin, Licong and Zrnic, Tijana},
  booktitle = 	 {Proceedings of the 41st International Conference on Machine Learning},
  pages = 	 {30546--30565},
  year = 	 {2024},
  editor = 	 {Salakhutdinov, Ruslan and Kolter, Zico and Heller, Katherine and Weller, Adrian and Oliver, Nuria and Scarlett, Jonathan and Berkenkamp, Felix},
  volume = 	 {235},
  series = 	 {Proceedings of Machine Learning Research},
  month = 	 {21--27 Jul},
  publisher =    {PMLR},
}

@article{xu2025learning,
  title={Learning Global {N}ash Equilibrium in Team Competitive Games with Generalized Fictitious Cross-Play},
  author={Xu, Zelai and Yu, Chao and Liang, Yancheng and Wu, Yi and Wang, Yu},
  journal={Journal of Machine Learning Research},
  volume={26},
  number={44},
  pages={1--30},
  year={2025}
}

@article{mitra2021linear,
  title={Linear convergence in federated learning: Tackling client heterogeneity and sparse gradients},
  author={Mitra, Aritra and Jaafar, Rayana and Pappas, George J and Hassani, Hamed},
  journal={Advances in Neural Information Processing Systems},
  volume={34},
  pages={14606--14619},
  year={2021}
}

@article{koloskova2022sharper,
  title={Sharper convergence guarantees for asynchronous {SGD} for distributed and federated learning},
  author={Koloskova, Anastasiia and Stich, Sebastian U and Jaggi, Martin},
  journal={Advances in Neural Information Processing Systems},
  volume={35},
  pages={17202--17215},
  year={2022}
}

@article{wang2021cooperative,
  title={Cooperative {SGD}: A unified framework for the design and analysis of local-update {SGD} algorithms},
  author={Wang, Jianyu and Joshi, Gauri},
  journal={Journal of Machine Learning Research},
  volume={22},
  number={213},
  pages={1--50},
  year={2021}
}

@article{ren2015faster,
  title={Faster {R-CNN}: Towards Real-Time Object Detection with Region Proposal Networks},
  author={Ren, Shaoqing and He, Kaiming and Girshick, Ross and Sun, Jian},
  journal={Advances in Neural Information Processing Systems},
  volume={28},
  year={2015}
}

@inproceedings{sessa2019bounding,
  title={Bounding inefficiency of equilibria in continuous actions games using submodularity and curvature},
  author={Sessa, Pier Giuseppe and Kamgarpour, Maryam and Krause, Andreas},
  booktitle={The 22nd International Conference on Artificial Intelligence and Statistics},
  pages={2017--2027},
  year={2019},
  organization={PMLR}
}

@article{roughgarden2015intrinsic,
  title={Intrinsic robustness of the price of anarchy},
  author={Roughgarden, Tim},
  journal={Journal of the ACM (JACM)},
  volume={62},
  number={5},
  pages={1--42},
  year={2015},
  publisher={ACM New York, NY, USA}
}

@article{maurer2006bounds,
  title={Bounds for linear multi-task learning},
  author={Maurer, Andreas},
  journal={Journal of Machine Learning Research},
  volume={7},
  pages={117--139},
  year={2006},
  publisher={JMLR. org}
}

@article{heliou2017learning,
  title={Learning with bandit feedback in potential games},
  author={H{\'e}liou, Am{\'e}lie and Cohen, Johanne and Mertikopoulos, Panayotis},
  journal={Advances in Neural Information Processing Systems},
  volume={30},
  year={2017}
}

@article{bistritz2022no,
  title={No weighted-regret learning in adversarial bandits with delays},
  author={Bistritz, Ilai and Zhou, Zhengyuan and Chen, Xi and Bambos, Nicholas and Blanchet, Jose},
  journal={Journal of Machine Learning Research},
  volume={23},
  number={139},
  pages={1--43},
  year={2022}
}

\end{document}